\documentclass{article}

\PassOptionsToPackage{table}{xcolor}
\usepackage{report}

\usepackage[utf8]{inputenc} 
\usepackage[T1]{fontenc}    
\usepackage{hyperref}       
\usepackage{url}            
\usepackage{booktabs}       
\usepackage{amsfonts}       
\usepackage{nicefrac}       
\usepackage{microtype}      
\usepackage{xcolor}         

\usepackage{mathtools}

\usepackage{latexsym}
\usepackage{graphicx}
\usepackage{amsmath}
\usepackage{amssymb}
\usepackage{booktabs}
\usepackage{multirow}
\usepackage{array}
\usepackage{tikz}
\usepackage{pgfplots}
\pgfplotsset{compat=1.18}
\usepgfplotslibrary{fillbetween}
\usetikzlibrary{positioning,arrows.meta,calc,shapes.geometric,patterns,fit,backgrounds,decorations.pathreplacing,shapes}
\usepackage{float}
\usepackage[font=small]{caption}
\usepackage{capt-of}
\usepackage{hyperref}
\usepackage{geometry}
\usepackage{natbib}
\usepackage{enumitem}
\usepackage{xcolor}
\usepackage{dblfloatfix}
\usepackage{placeins}
\usepackage{algorithm}
\usepackage{algorithmic}
\usepackage{wrapfig} 
\usepackage{tcolorbox}
\tcbuselibrary{breakable}

\definecolor{cBlue}{HTML}{3B82F6}
\definecolor{cGreen}{HTML}{10B981}
\definecolor{cOrange}{HTML}{F59E0B}
\definecolor{cRed}{HTML}{EF4444}
\definecolor{cPurple}{HTML}{8B5CF6}
\definecolor{cGray}{HTML}{6B7280}
\definecolor{cDark}{HTML}{1F2937}
\definecolor{cLight}{HTML}{F3F4F6}
\definecolor{njuPurple}{RGB}{220,205,230}
\definecolor{njuPurpleLight}{RGB}{250,245,252}

\newtcolorbox{abstractbox}{
    colback=njuPurpleLight,
    colframe=njuPurple,
    boxrule=1pt,
    arc=4mm,
    left=8pt,
    right=8pt,
    top=8pt,
    bottom=8pt,
    opacityback=0.95,
    breakable
}

\usepackage{amsthm}
\newtheorem{theorem}{Theorem}[section]   
\newtheorem{lemma}[theorem]{Lemma}

\newtheorem{assumption}{Assumption}[section]
\newtheorem{proposition}{Proposition}[section]   

\usepackage{tikz}
\usepackage{pgfplots}
\usepgfplotslibrary{groupplots}
\pgfplotsset{compat=1.18}

\usepackage{etoolbox}

\let\oldmathbb\mathbb

\renewcommand{\mathbb}[1]{%
  \ifstrequal{#1}{1}{\mathbf{1}}{\oldmathbb{#1}}%
}
\title{From Solver Feedback to Faithful Plans: Multi-Role Reinforcement Learning for Symbolic Planning}

\author{%
\begin{tabular}{cccc}
Chenghao Zhang$^{1}$ & Yikai Mao$^{1}$ & Haoyu Gao$^{2}$ & Saisai Hu$^{3}$
\end{tabular}\\[3pt]
\begin{tabular}{ccc}
Yuxi Cheng$^{4}$ & Shanqi Liu$^{4}$ & Dan Roth$^{1}$
\end{tabular}
}

\begin{document}

\maketitle

\let\oldthefootnote\thefootnote
\let\thefootnote\relax
\footnotetext{$^{1}$~University of Pennsylvania.}
\footnotetext{$^{2}$~Georgia Institute of Technology.}
\footnotetext{$^{3}$~Pace University.}
\footnotetext{$^{4}$~Unaffiliated.}
\let\thefootnote\oldthefootnote

\begin{abstract}
\begin{abstractbox}
Reliable planning requires converting natural-language instructions into executable symbolic specifications, yet large language models remain brittle without costly PDDL annotations and may exploit solver success in semantically unfaithful ways.
We study how to learn faithful natural-language-to-PDDL formalization using only solver feedback, without human-written demonstrations.
We propose a solver-grounded multi-role reinforcement learning framework where a single language model acts as an Actor, Judge, and Editor for generation, verification, and repair.
The Actor proposes PDDL specifications, the Judge provides a solver-calibrated quality signal, and the Editor performs bounded diagnostic-conditioned refinement.
On PlanBench, our method improves average success from 35.5\% for LLM+P to 70.8\%, achieves 66.3\% faithful success, and reduces semantic drift to 6.4\%.
These results show that organizing solver feedback into generation, verification, and repair roles enables more scalable and faithful annotation-free symbolic planning.
\end{abstractbox}
\end{abstract}

\section{Introduction}
\label{sec:introduction}

Structured planning is a central capability for intelligent agents that must transform high-level instructions into executable action models under constraints. 
It is critical for robotics, workflow automation, logistics, and interactive decision-making, where success depends not only on producing plausible language but also on satisfying explicit state-transition semantics. 
Large language models (LLMs) have shown impressive general-purpose reasoning and instruction-following abilities \citep{brown2020language,openai2023gpt4,touvron2023llama,lin2026cec}, yet reliable planning remains difficult because valid plans require logical consistency, long-horizon state tracking, and faithful grounding in task constraints.

Recent benchmarks have made this limitation increasingly clear. 
PlanBench shows that LLMs often fail on classical planning domains, especially when predicate names are obfuscated and surface lexical cues are removed \citep{valmeekam2024planbench}. 
This supports the broader view that LLMs often behave as approximate retrievers rather than systematic planners over state transitions \citep{kambhampati2024can}. 
Even reasoning-enhanced models and chain-of-thought prompting remain brittle under planning-specific perturbations, suggesting that fluent intermediate reasoning is not equivalent to executable planning \citep{stechly2024cot,valmeekam2024lrm}. 
NATURAL PLAN further shows that realistic natural-language planning becomes sharply harder as constraint complexity increases, and self-correction does not reliably recover valid solutions \citep{zheng2024naturalplan}.

A promising response is to use LLMs not as direct planners, but as formalizers that translate natural-language tasks into symbolic representations such as Planning Domain Definition Language (PDDL), after which an external solver performs the actual planning. 
LLM+P demonstrates the value of this neuro-symbolic decomposition by delegating search to a classical planner \citep{liu2023llmp}, and subsequent work explores PDDL goal translation, generalized planning programs, and environment-aided PDDL construction \citep{xie2023goals,silver2024generalized,mahdavi2024leveragingenvironmentinteractionautomated}. 
However, these approaches still face a fundamental supervision bottleneck: high-quality natural-language-to-PDDL formalization typically requires curated demonstrations, fixed domain assumptions, or supervised initialization. 
Instruction-tuning methods such as PDDL-Instruct improve symbolic planning performance, but their reliance on annotated examples limits scalability to new domains and language distributions \citep{verma2025pddlinstruct}.

\begin{figure}[t]
    \centering
    \includegraphics[width=0.52\linewidth]{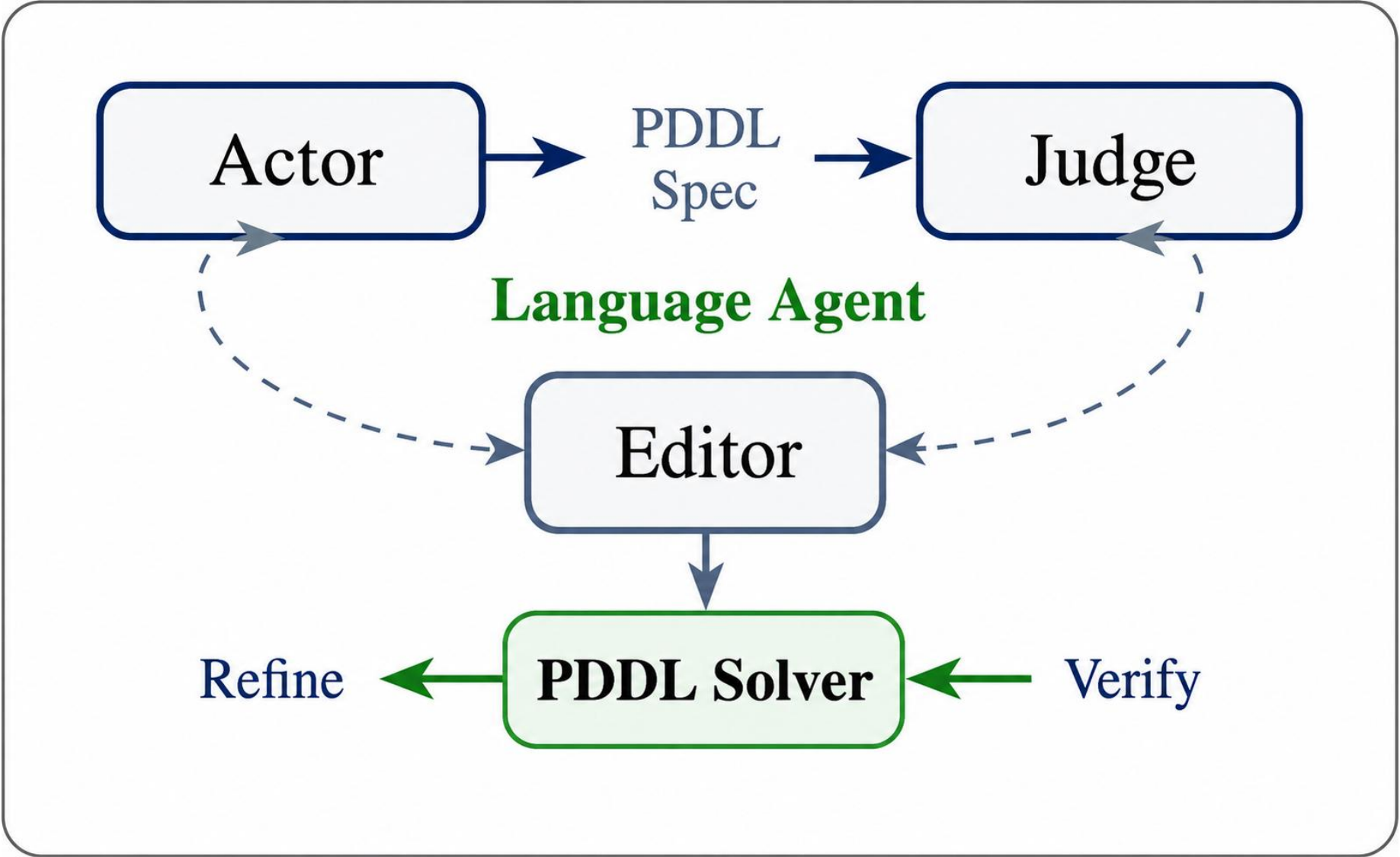}
    \caption{One Brain Three Roles framework. A single LLM plays Actor, Judge, and Editor roles with shared parameters, grounded by deterministic PDDL solver feedback.}
    \label{fig:framework}
\end{figure}

The deeper challenge is that solver feedback alone is not automatically a faithful learning signal. 
A PDDL solver can verify whether a generated specification is syntactically valid and executable, but solver success is defined with respect to the generated formal specification rather than the original natural-language intent. 
Thus, optimizing only for solver success can reward underspecified or semantically shifted formalizations that are easy to solve but no longer faithful to the task. 
This creates a high-level conflict between \emph{annotation-free learning} and \emph{semantic faithfulness}: the former demands replacing human labels with environment feedback, while the latter requires preventing the model from exploiting imperfections or ambiguities in that feedback. 
Similar concerns arise in broader reinforcement learning settings, where optimizing a proxy reward can induce reward gaming when the reward does not fully capture the intended objective \citep{skalse2022defining}.

This paper asks: \emph{Can an LLM learn reliable symbolic planning from solver feedback alone, without any human-annotated PDDL demonstrations?} 
We address this question by treating the symbolic solver as the only grounded oracle while organizing learning around generation, verification, and repair. 
At a high level, the model learns not only to produce a candidate formal specification, but also to evaluate solver-grounded correctness and refine failed specifications using executable diagnostics, thereby converting sparse solver outcomes into a more informative training process.

We propose \textbf{Solver-Grounded Multi-Role Reinforcement Learning}, a zero-annotation framework for natural-language-to-PDDL planning. 
A single LLM is conditioned to play three coordinated roles: an Actor that proposes formal specifications, a Judge that learns a solver-calibrated quality signal, and an Editor that performs bounded repair from solver diagnostics. 
This design preserves the scalability of solver-only supervision while reducing the risk that optimization collapses into purely solver-facing shortcuts.

Our contributions are as follows:
\begin{itemize}
    \item We identify a key obstacle in annotation-free symbolic planning: raw solver success is a necessary but insufficient learning signal because it can diverge from semantic faithfulness to the original natural-language task. This reframes NL-to-PDDL learning as a problem of solver-grounded optimization under potential specification drift.
    
    \item We introduce a multi-role reinforcement learning framework in which generation, verification, and repair are learned jointly from deterministic solver feedback, without human-annotated PDDL demonstrations. The shared-role design enables verifier co-evolution and diagnostic-conditioned refinement while keeping most model parameters coupled across roles.
    
    \item Across PlanBench domains, our method achieves strong in-domain planning success, including substantial gains on obfuscated and structurally challenging settings. It also transfers zero-shot to ProntoQA and NATURAL PLAN, indicating that solver-grounded multi-role learning improves not only benchmark-specific PDDL synthesis but also broader planning and constraint-satisfaction behavior.
\end{itemize}

\section{Related Work}
\label{sec:related}

\paragraph{LLMs as planners.}
A first line asks whether scaling and prompting alone are sufficient for planning. 
Chain-of-thought and search-based deliberation can improve intermediate reasoning \citep{wei2022chain,yao2023tree,hao2023rap}, but planning benchmarks show that fluent rationales do not guarantee executable state-transition reasoning. 
PlanBench and follow-up analyses reveal large gaps between textual plausibility and valid planning behavior, especially under predicate obfuscation, longer horizons, and reasoning-model variants \citep{valmeekam2024planbench,kambhampati2024can,stechly2024cot,valmeekam2024lrm}. 
NATURAL PLAN further shows that realistic language planning degrades sharply as constraints become more complex \citep{zheng2024naturalplan}. 
These works motivate execution-grounded evaluation, whereas we use solver feedback as the training signal, optimizing planning competence through repeated interaction with a symbolic verifier.

\paragraph{LLM-as-formalizer and neuro-symbolic planning.}
A second line shifts from direct plan generation to inducing formal representations that planners can solve. 
LLM+P demonstrates this paradigm by translating natural language into PDDL, invoking a classical planner, and verbalizing the resulting plan \citep{liu2023llmp}; subsequent work studies goal translation, generalized planning programs, domain generation, and text-to-PDDL evaluation \citep{xie2023goals,silver2024generalized,oswald2024planningdomain,zuo2025planetarium,zhang2024proc2pddl}. 
Recent studies show that LLM-as-formalizer can outperform direct planning, but still suffers from semantic omissions, naturalness shift, and scaling failures in large formal structures \citep{huang2025limitlanguagemodelsplanning,jiang2026plannersformalizers}. 
Supervised methods such as PDDL-Instruct improve performance with annotated instruction tuning \citep{verma2025pddlinstruct}, while environment-interaction methods refine PDDL through feedback \citep{mahdavi2024leveragingenvironmentinteractionautomated}. 
In contrast, we formulate NL-to-PDDL induction as annotation-free reinforcement learning, where an Actor generates specifications, a Judge learns solver-calibrated scoring, and an Editor performs diagnostic-conditioned repair.

\paragraph{Feedback-driven self-improvement and verifiable rewards.}
A third line explores feedback-driven LLM self-improvement without dense human supervision. 
Reflexion, Self-Refine, ISR-LLM, and AdaPlanner use critique--revise or environment-feedback loops \citep{shinn2023reflexion,madaan2023selfrefine,zhou2024isrllm,sun2023adaplanner}, while LLM-as-a-judge methods reduce annotation cost but risk evaluator bias and reward hacking without external grounding \citep{gu2025surveyllmasajudge,li2025generationjudgmentopportunitieschallenges}. 
Recent RLVR and self-play systems show that verifiable feedback can scale reasoning training \citep{shao2024deepseekmath,deepseek2025r1,zhao2025absolutezeroreinforcedselfplay,chen2025multiagentevolvellmselfimprove}, but they mainly target curated tasks or code/math verifiers and do not address the gap between solver success and semantic faithfulness. 
Our work brings verifiable-reward learning to symbolic planning by grounding Actor, Judge, and Editor roles in a deterministic PDDL solver, converting sparse executability feedback into calibrated verification and repair signals.

\section{Method}
\label{sec:method}

\begin{figure*}[t]
\centering
\includegraphics[width=\textwidth]{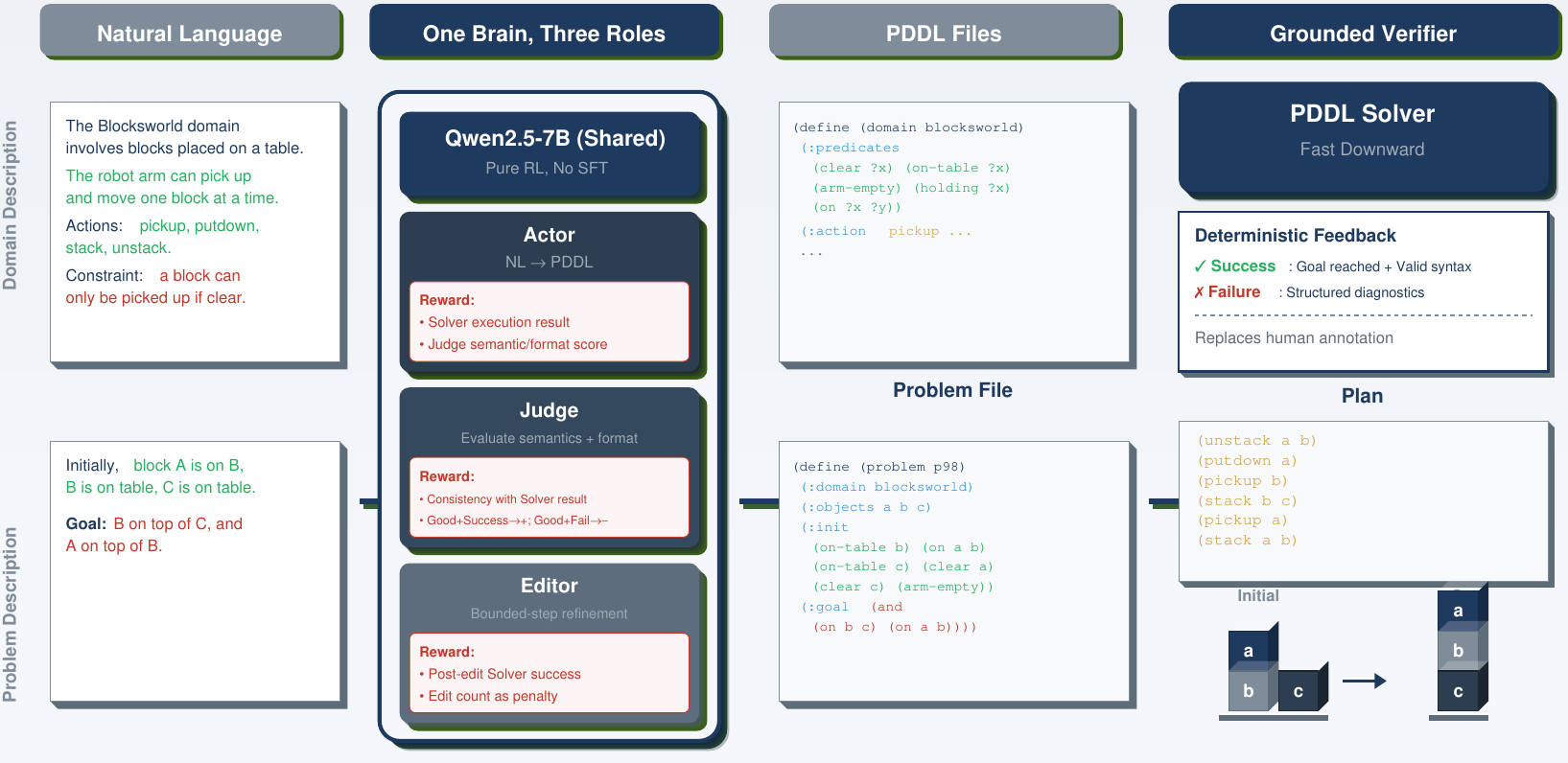}
\caption{\textbf{Solver-Grounded Multi-Role RL for NL-to-PDDL Planning.}
A single shared LLM backbone is conditioned into Actor, Judge, and Editor roles. The Actor generates an initial PDDL specification. A deterministic PDDL Solver verifies executability, producing a success label and diagnostics  on failure. The Judge predicts a calibrated solver score, and the Editor uses diagnostics to perform bounded local edits, feeding back to the solver.}
\label{fig:architecture}
\end{figure*}

\subsection{Task and Solver-Grounded Environment}
\label{sec:task-env}

We study natural-language-to-PDDL specification induction without human annotations. Each task is a natural-language planning description \(x\sim\mathcal{P}_{\mathrm{data}}\). The model outputs a PDDL specification
$
  y=(y^{\mathrm{dom}},y^{\mathrm{prob}}),
$
consisting of a domain file and a problem file. A deterministic PDDL environment \(\mathcal{E}\), implemented with Fast Downward, verifies the generated specification and returns
\begin{equation}
  (b,d)=\mathcal{E}(y),
  \qquad
  b\in\{0,1\},
  \label{eq:solver-feedback}
\end{equation}
where \(b=1\) indicates that the specification is syntactically valid and solver-executable within the time limit, and \(d\) contains structured diagnostics such as parse errors, type mismatches, unsatisfied preconditions, unreachable goals, and failed action traces.

The central challenge is that solver success is not identical to task-level semantic faithfulness. Because the model controls the generated specification, optimizing only \(b\) can reward degenerate specifications that weaken goals, remove constraints, or alter predicates to make the instance easier than the original task. Our method keeps the solver as the only external feedback source, but separates the learning problem into generation, calibration, and repair roles so that sparse solver outcomes become more useful for policy optimization.

\subsection{Role-Conditioned Multi-Role Policy}
\label{sec:role-policy}

We use a single language model with shared backbone parameters and lightweight role-specific parameters. Let
\(\rho\in\{\mathrm{A},\mathrm{J},\mathrm{E}\}\) denote the role identifier for Actor, Judge, and Editor. Each role has a learned embedding \(e_\rho\), and the role-conditioned model defines
\begin{equation}
  \pi_\theta(\cdot\mid c,\rho)
  =
  \mathrm{LM}_{\theta}\!\left(\cdot\mid [e_\rho;c]\right),
  \label{eq:role-policy}
\end{equation}
where \(c\) is the role-specific context. The Actor context is \(x\), the Judge context is \((x,y)\), and the Editor context is \((x,y_t,d_t)\), where \(y_t\) is the current specification and \(d_t\) is the latest solver diagnostic.

The parameters are partitioned as
\begin{equation}
  \theta=\theta_{\mathrm{sh}}\cup\theta_{\mathrm{A}}\cup\theta_{\mathrm{J}}\cup\theta_{\mathrm{E}} .
\end{equation}
The shared backbone \(\theta_{\mathrm{sh}}\) contains approximately \(95\%\) of the parameters, while the remaining parameters are lightweight role-specific heads. This design allows each role to specialize while keeping most representation learning coupled across generation, verification, and repair.

\subsection{Solver-Grounded Training Episode}
\label{sec:training-episode}

Algorithm~\ref{alg:training-episode} summarizes one training episode. The Actor first samples an initial PDDL specification. The solver verifies it and returns a binary success label with diagnostics. The Judge predicts a scalar score for the generated specification. If the initial specification fails, the Editor performs at most \(H_{\max}\) repair steps, each conditioned on the task, the current specification, and the latest diagnostic. The episode stops when the solver succeeds or the repair budget is exhausted.

\begin{algorithm}[t]
\caption{One solver-grounded multi-role training episode}
\label{alg:training-episode}
\begin{algorithmic}[1]
\REQUIRE Task \(x\), solver \(\mathcal{E}\), maximum repair horizon \(H_{\max}\)
\STATE Sample initial specification \(y_0\sim\pi_\theta(\cdot\mid x,\mathrm{A})\)
\STATE Query solver \((b_0,d_0)=\mathcal{E}(y_0)\)
\STATE Compute Judge score \(s_J(x,y_0)\)
\STATE Set \(T\gets 0\), \(y_T\gets y_0\), \(b_T\gets b_0\), \(d_T\gets d_0\)
\FOR{\(t=0,\ldots,H_{\max}-1\)}
  \IF{\(b_T=1\)}
    \STATE \textbf{break}
  \ENDIF
  \STATE Sample repaired specification \(y_{t+1}\sim\pi_\theta(\cdot\mid x,y_t,d_t,\mathrm{E})\)
  \STATE Query solver \((b_{t+1},d_{t+1})=\mathcal{E}(y_{t+1})\)
  \STATE Set \(T\gets t+1\), \(y_T\gets y_{t+1}\), \(b_T\gets b_{t+1}\), \(d_T\gets d_{t+1}\)
\ENDFOR
\STATE Update Actor, Judge, and Editor using Eq.~\eqref{eq:role-objectives} and Eq.~\eqref{eq:total-loss}
\end{algorithmic}
\end{algorithm}

At inference time, the same procedure is used without gradient updates. The system returns the first solver-executable specification found within the repair horizon; if no repair succeeds, it returns the final edited specification for evaluation. Decoding and optimization details are given in Appendix~\ref{app:training-details}.

\subsection{Role-Specific Learning Signals}
\label{sec:role-objectives}

The Actor learns global specification generation, the Judge learns solver-calibrated verification, and the Editor learns diagnostic-conditioned repair. The Judge score is produced by a scalar head on top of the shared representation:
\begin{equation}
  s_J(x,y)
  =
  \sigma\!\left(g_\theta([e_{\mathrm{J}};x;y])\right)
  \in(0,1),
  \label{eq:judge-score}
\end{equation}
where \(g_\theta\) is the Judge logit head and \(\sigma\) is the sigmoid function.

For an episode with initial output \(y_0\), initial solver label \(b_0\), final output \(y_T\), final solver label \(b_T\), and repair length \(T\), the role-specific learning signals are
\begin{equation}
\begin{aligned}
  R_{\mathrm{A}}(x,y_0)
  &=
  b_0+\lambda_{\mathrm{J}}\,s_J(x,y_0),\\
  \mathcal{L}_{\mathrm{J}}(x,y_0,b_0)
  &=
  \mathrm{BCE}\!\left(s_J(x,y_0),b_0\right)
  +\lambda_{\mathrm{band}}\ell_{\mathrm{band}}\!\left(s_J(x,y_0),b_0\right),\\
  R_{\mathrm{E}}(x,y_{0:T})
  &=
  b_T-\beta T .
\end{aligned}
\label{eq:role-objectives}
\end{equation}
The Actor reward combines the initial solver label with Judge-based shaping. The Judge loss trains \(s_J\) to predict solver success and penalizes over-confident false positives. The Editor reward favors successful repairs with fewer editing steps.

Actor and Editor token policies are optimized with PPO, while the Judge is optimized as a calibrated binary predictor. The total loss is
\begin{equation}
  \mathcal{L}_{\mathrm{total}}
  =
  \mathcal{L}_{\mathrm{PPO}}^{\mathrm{A}}
  +\lambda_{\mathrm{E}}\mathcal{L}_{\mathrm{PPO}}^{\mathrm{E}}
  +\lambda_{\mathrm{J}}\mathcal{L}_{\mathrm{J}} .
  \label{eq:total-loss}
\end{equation}
Appendix~\ref{app:training-details} specifies the PPO objective, banded Judge penalty, decoding policy, and hyperparameters.

\subsection{Judge Calibration and Correctness Connection}
\label{sec:judge-calibration}

The Judge is trained against solver outcomes rather than human semantic labels. Its direct target is therefore solver-verified executability. The connection to task-level correctness depends on how often solver success agrees with a reference correctness notion on the policy-induced distribution. Let \(B\in\{0,1\}\) denote solver success and \(C\in\{0,1\}\) denote task-level correctness for a generated pair \((X,Y)\). If the on-policy disagreement rate satisfies \(\Pr(B\neq C)\le\varepsilon\), and the learned Judge has calibration error \(\delta_J\) with respect to the Bayes-optimal solver-success predictor, then
\begin{equation}
  \mathbb{E}_{(X,Y)}
  \left[
    \left|s_J(X,Y)-C(X,Y)\right|
  \right]
  \le
  \varepsilon+\delta_J .
  \label{eq:judge-correctness-main}
\end{equation}
Thus, the Judge provides a reliable shaping signal when solver success and task-level correctness have limited disagreement on generated samples. Appendix~\ref{app:judge-analysis} proves Eq.~\eqref{eq:judge-correctness-main}, gives the corresponding ranking guarantee, and describes how to estimate the disagreement rate.

\subsection{Why Separate Actor, Judge, and Editor Roles}
\label{sec:why-roles}

The three roles address different sources of difficulty in annotation-free symbolic planning. The Actor performs a global mapping from natural language to a complete PDDL domain--problem pair. The Judge converts sparse solver outcomes into a calibrated score that can guide learning. The Editor uses solver diagnostics to make bounded local repairs after failure. A single final-outcome policy can leave the initial specification weakly constrained once local repair becomes strong, because final solver success may no longer distinguish good initial formalizations from poor but repairable ones. Appendix~\ref{app:actor-editor-analysis} formalizes this failure mode.

\subsection{Parameter Sharing and Reward-Hacking Directions}
\label{sec:sharing-theory}

We now analyze why the shared-backbone design is less permissive than three fully separate role models in a local reward-hacking sense. Let \(J_{\mathrm{A}}(\theta)\) denote the Actor objective and let \(C(\theta)\) denote expected task-level correctness. Around a reference point \(\theta_0\), define the local reward-hacking gradient
\begin{equation}
  g_{\mathrm{hack}}(\theta_0)
  =
  \nabla_{\theta}J_{\mathrm{A}}(\theta_0)
  -
  a\nabla_{\theta}C(\theta_0),
  \qquad a>0 .
  \label{eq:ghack-main}
\end{equation}
For a unit-norm perturbation \(\Delta\theta\), the first-order Actor gain not explained by correctness is
\[
  \Delta_{\mathrm{hack}}(\theta_0;\Delta\theta)
  =
  g_{\mathrm{hack}}(\theta_0)^\top \Delta\theta .
\]
The admissible perturbation set depends on the architecture. With three separate models, the Actor can move in its own \(P_{\mathrm{A}}\)-dimensional parameter space without directly changing Judge or Editor behavior. With a shared backbone, Judge- and Editor-neutral perturbations are restricted to the small role-private subspace, provided the shared backbone directions are observed by Judge and Editor objectives. Appendix~\ref{app:sharing-analysis} states the assumptions and proof.

\begin{theorem}[Reward hacking scaling with parameter count]
\label{thm:hacking-scaling-main}
Under the local isotropic-gradient model and the shared-backbone constraint in Appendix~\ref{app:sharing-analysis}, the expected worst-case first-order reward-hacking gain satisfies:
\begin{enumerate}
  \item \textbf{Three fully separate models.}
  If the Actor has \(P_{\mathrm{A}}\) parameters, then
  \begin{equation}
    \mathbb{E}
    \left[
      \sup_{\Delta\theta\in\mathcal{H}_{\mathrm{sep}}}
      \Delta_{\mathrm{hack}}(\theta_0;\Delta\theta)
    \right]
    \asymp
    \sigma\sqrt{P_{\mathrm{A}}}.
    \label{eq:sep-scaling-main}
  \end{equation}

  \item \textbf{Shared-backbone architecture.}
  If the Judge- and Editor-neutral perturbations are contained in an Actor head of dimension \(h\), then
  \begin{equation}
    \mathbb{E}
    \left[
      \sup_{\Delta\theta\in\mathcal{H}_{\mathrm{shr}}}
      \Delta_{\mathrm{hack}}(\theta_0;\Delta\theta)
    \right]
    \lesssim
    \sigma\sqrt{h}.
    \label{eq:shr-scaling-main}
  \end{equation}
\end{enumerate}
\end{theorem}

The theorem shows that separate models allow reward-hacking gains to grow with the Actor parameter count. 
By contrast, a shared backbone exposes most directions to Judge and Editor objectives, leaving only the small Actor head as the main role-private subspace. 
Thus, parameter sharing does not eliminate reward hacking, but reduces the local degrees of freedom for improving Actor reward without corresponding verification and repair changes.

\begin{table}[t]
\centering
\small
\setlength{\tabcolsep}{3.2pt}
\caption{
Planning success rate (\%) on PlanBench and zero-shot transfer benchmarks.
BW denotes BlocksWorld and MBW denotes Mystery BlocksWorld.
Avg. is the mean over the four PlanBench domains.
}
\label{tab:main-and-transfer}
\resizebox{0.9\textwidth}{!}{
\begin{tabular}{@{}lccccc|ccc@{}}
\toprule
\textbf{Method}
& \textbf{BW}
& \textbf{MBW}
& \textbf{Logistics}
& \textbf{Gripper}
& \textbf{Avg.}
& \textbf{ProntoQA}
& \textbf{Trip}
& \textbf{Calendar} \\
\midrule
LLM$^{\text{CoT}}$
& \(24.0{\scriptstyle\pm4.3}\)
& \(0.0{\scriptstyle\pm0.0}\)
& \(3.0{\scriptstyle\pm1.7}\)
& \(23.0{\scriptstyle\pm4.2}\)
& \(12.5{\scriptstyle\pm1.6}\)
& \(66.0{\scriptstyle\pm4.7}\)
& \(8.0{\scriptstyle\pm2.7}\)
& \(34.0{\scriptstyle\pm4.7}\) \\
LLM$^{\text{ToT}}$
& \(3.0{\scriptstyle\pm1.7}\)
& \(3.0{\scriptstyle\pm1.7}\)
& \(7.0{\scriptstyle\pm2.6}\)
& \(13.0{\scriptstyle\pm3.4}\)
& \(6.5{\scriptstyle\pm1.2}\)
& -- & -- & -- \\
LLM+P
& \(87.0{\scriptstyle\pm3.4}\)
& \(37.0{\scriptstyle\pm4.8}\)
& \(18.0{\scriptstyle\pm3.8}\)
& \(0.0{\scriptstyle\pm0.0}\)
& \(35.5{\scriptstyle\pm1.8}\)
& \(4.0{\scriptstyle\pm2.0}\)
& \(16.0{\scriptstyle\pm3.7}\)
& \(12.0{\scriptstyle\pm3.2}\) \\
\midrule
\rowcolor{gray!20}
\textbf{Ours}
& \(\mathbf{98.0}{\scriptstyle\pm1.4}\)
& \(\mathbf{71.0}{\scriptstyle\pm4.5}\)
& \(\mathbf{58.0}{\scriptstyle\pm4.9}\)
& \(\mathbf{56.0}{\scriptstyle\pm5.0}\)
& \(\mathbf{70.8}{\scriptstyle\pm2.1}\)
& \(\mathbf{90.0}{\scriptstyle\pm3.0}\)
& \(\mathbf{48.0}{\scriptstyle\pm5.0}\)
& \(\mathbf{40.0}{\scriptstyle\pm4.9}\) \\
\bottomrule
\end{tabular}
}
\end{table}

\section{Experiments}

\subsection{Experimental Setup}
\label{sec:exp-setup}

\textbf{Benchmarks.}
We evaluate in-domain planning performance on PlanBench~\citep{valmeekam2024planbench}, using four PDDL-derived domains: BlocksWorld, Mystery BlocksWorld, Logistics, and Gripper. Mystery BlocksWorld obfuscates object and predicate names, making it a diagnostic test of planning beyond lexical memorization. We further evaluate zero-shot transfer on ProntoQA~\citep{saparov2023languagemodelsgreedyreasoners} and the Trip Planning and Calendar Scheduling subsets of NATURAL PLAN~\citep{zheng2024naturalplan}. Detailed benchmark descriptions are provided in Appendix~\ref{app:experimental-details}.

\textbf{Evaluation.}
For PlanBench, generated domain--problem pairs are evaluated with Fast Downward under a 60-second timeout. An instance is counted as solved only if the generated PDDL is syntactically valid and the returned plan satisfies the goal conditions. For ProntoQA and NATURAL PLAN, we follow the original evaluation protocol and report exact-match success.

\textbf{Model.}
Our method is built on Qwen2.5-7B~\citep{qwen2025qwen25technicalreport} initialized from pretrained weights. We use no human-annotated PDDL demonstrations and no supervised fine-tuning on planning data; learning is driven only by solver-grounded multi-role reinforcement signals described in Section~\ref{sec:role-objectives}.

\textbf{Baselines.}
We compare with three zero-annotation baselines: Chain-of-Thought prompting~\citep{wei2022chain}, Tree-of-Thought search~\citep{yao2023tree}, and LLM+P~\citep{liu2023llmp}. Unless otherwise stated, these baselines use GPT-4o as the underlying model. Appendix~\ref{app:experimental-details} gives the full baseline configurations.

\subsection{Main results}
\label{sec:results-analysis}

\textbf{Overall planning performance and zero-shot transfer.}
Table~\ref{tab:main-and-transfer} summarizes the main in-domain and out-of-domain results. 
Our method achieves the best performance on all PlanBench domains, with an average success rate of \(70.8\%\), compared with \(35.5\%\) for LLM+P and much lower averages for prompting-only baselines. 
The gains are largest on Mystery BlocksWorld, Logistics, and Gripper, where lexical shortcuts are less useful and small symbolic errors can invalidate the whole plan. 
The same model also transfers zero-shot to ProntoQA and NATURAL PLAN, suggesting that solver-grounded multi-role training improves general constraint satisfaction rather than only memorizing PlanBench-specific formats. 
Additional discussion is provided in Appendix~\ref{app:main-result-discussion}.

\textbf{Controlled comparison under matched backbone and solver budget.}
To rule out the possibility that the gains come from a stronger backbone or more solver access, Table~\ref{tab:controlled-comparison-main} compares all methods under a matched \(K=6\) solver-call budget. 
Our method still outperforms the strongest same-backbone baseline by \(21.0\) points in average PlanBench success, while using fewer solver calls on average. 
It also achieves the highest faithful success and the lowest semantic drift, indicating that trained role specialization is more effective than prompting-based diagnostic revision alone. 
The controlled protocol is detailed in Appendix~\ref{app:controlled-protocol}.

\begin{table}[t]
\centering
\small
\caption{
Controlled comparison under matched solver-call budget. 
All methods are allowed at most \(K=6\) solver calls per test instance. 
Faithful denotes solver-successful and semantically faithful outputs. 
Drift is the conditional fraction of solver-successful outputs that fail semantic checking.
}
\label{tab:controlled-comparison-main}
\resizebox{0.92\textwidth}{!}{
\begin{tabular}{@{}lcccccccccc@{}}
\toprule
\textbf{Method}
& \textbf{Backbone}
& \textbf{BW}
& \textbf{MBW}
& \textbf{Logistics}
& \textbf{Gripper}
& \textbf{Avg.}
& \textbf{Avg. Calls}
& \textbf{Faithful}
& \textbf{Drift} \(\downarrow\) \\
\midrule
Qwen-CoT 
& Qwen2.5-7B 
& 31 & 1 & 5 & 25 & 15.5 
& 5.8 & 11.6 & 25.2 \\
Qwen-ToT 
& Qwen2.5-7B 
& 16 & 5 & 9 & 18 & 12.0 
& 5.9 & 9.1 & 24.4 \\
Qwen-LLM+P 
& Qwen2.5-7B 
& 80 & 32 & 20 & 7 & 34.8 
& 5.2 & 27.1 & 22.1 \\
Qwen-Self-Refine+Solver 
& Qwen2.5-7B 
& 85 & 43 & 36 & 35 & 49.8 
& 4.7 & 37.5 & 24.7 \\
GPT-4o-Self-Refine+Solver 
& GPT-4o 
& 91 & 47 & 42 & 37 & 54.3 
& 4.5 & 43.0 & 20.8 \\
\midrule
\rowcolor{gray!20}
\textbf{Ours} 
& \textbf{Qwen2.5-7B} 
& \textbf{98} 
& \textbf{71} 
& \textbf{58} 
& \textbf{56} 
& \textbf{70.8}
& \textbf{3.2}
& \textbf{66.3}
& \textbf{6.4} \\
\bottomrule
\end{tabular}
}
\end{table}

\subsection{Analysis}
\textbf{Semantic faithfulness.}
Because a solver-executable PDDL specification can still deviate from the original task semantics, Table~\ref{tab:faithfulness-main} reports both solvability and faithful success. 
Our method achieves a small solvability--faithfulness gap, with \(70.8\%\) solvability and \(66.3\%\) faithful success, while reducing drift to \(6.4\%\). 
The per-domain faithful-success results show that the improvement persists across all PlanBench domains, especially on Mystery BlocksWorld and the structurally brittle Logistics and Gripper domains. 
The reference-checking protocol and metric definitions are given in Appendix~\ref{app:faithfulness-protocol}.

\begin{table}[t]
\centering
\small
\caption{
Semantic faithfulness evaluation on PlanBench. 
Left: aggregate faithfulness metrics. 
Right: per-domain faithful success rate. 
Solv. denotes solver success. 
Faithful requires both solver executability and semantic consistency with the reference task.
}
\label{tab:faithfulness-main}
\begin{minipage}{0.50\textwidth}
\centering
\textbf{(a) Aggregate faithfulness}\\[2pt]
\resizebox{\linewidth}{!}{
\begin{tabular}{@{}lccccc@{}}
\toprule
\textbf{Method} 
& \textbf{Solv.} 
& \textbf{Faithful} 
& \textbf{Drift} \(\downarrow\)
& \textbf{Goal} 
& \textbf{Schema-F1} \\
\midrule
LLM+P 
& 35.5 & 28.3 & 20.4 & 84.7 & 77.6 \\
Solver-only RL 
& 39.8 & 25.6 & 35.7 & 73.5 & 68.2 \\
Ours w/o Judge 
& 44.1 & 29.4 & 33.3 & 76.8 & 70.5 \\
\midrule
\rowcolor{gray!20}
\textbf{Ours} 
& \textbf{70.8} 
& \textbf{66.3} 
& \textbf{6.4} 
& \textbf{96.1} 
& \textbf{91.8} \\
\bottomrule
\end{tabular}
}
\end{minipage}
\hfill
\begin{minipage}{0.47\textwidth}
\centering
\textbf{(b) Per-domain faithful success}\\[2pt]
\resizebox{\linewidth}{!}{
\begin{tabular}{@{}lccccc@{}}
\toprule
\textbf{Method} 
& \textbf{BW} 
& \textbf{MBW} 
& \textbf{Logistics} 
& \textbf{Gripper} 
& \textbf{Avg.} \\
\midrule
LLM+P 
& 75 & 26 & 12 & 0 & 28.3 \\
Solver-only RL
& 51 & 9 & 22 & 20 & 25.6 \\
Ours w/o Judge
& 57 & 13 & 24 & 24 & 29.4 \\
\midrule
\rowcolor{gray!20}
\textbf{Ours} 
& \textbf{94} 
& \textbf{67} 
& \textbf{53} 
& \textbf{51} 
& \textbf{66.3} \\
\bottomrule
\end{tabular}
}
\end{minipage}
\end{table}

\begin{figure*}[ht]
    \centering
    \includegraphics[width=\textwidth]{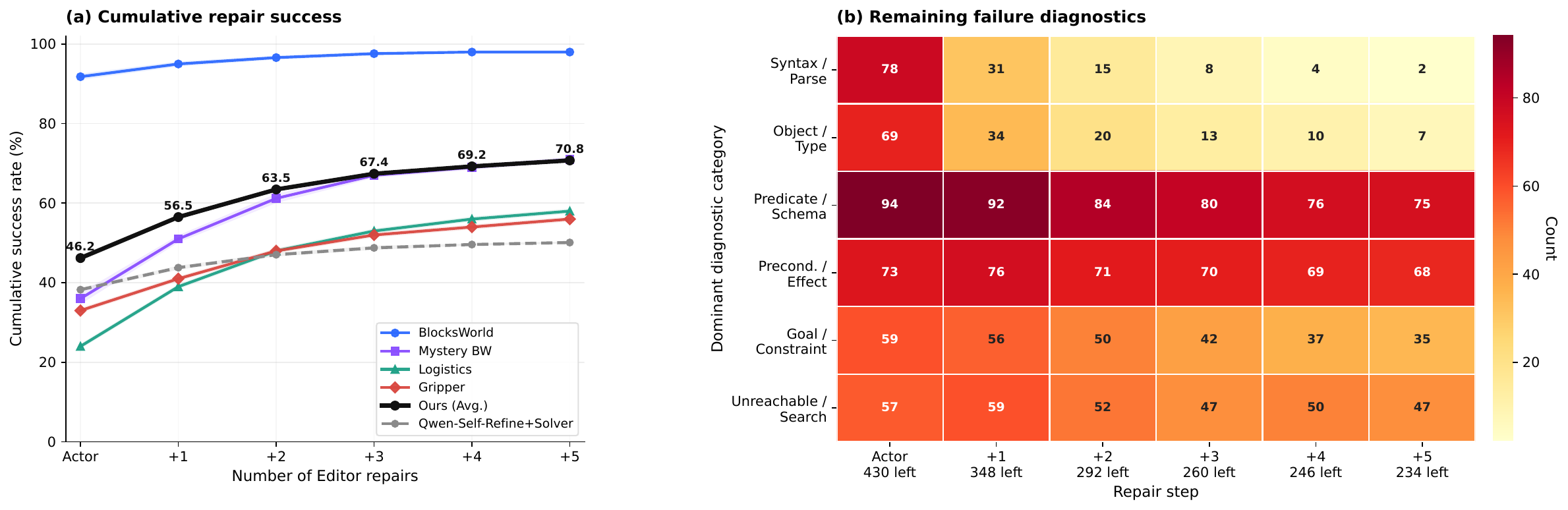}
    \vspace{-2mm}
    \caption{
    \textbf{Repair and diagnostic analysis.}
    Left: cumulative success rate over Editor repair steps. 
    Step \(0\) denotes the initial Actor output, and steps \(1\)--\(5\) denote successive repairs. 
    Right: diagnostic heatmap over remaining unsolved cases after each step.
    }
    \label{fig:repair-diagnostic-main}
\end{figure*}

\begin{table*}[ht]
\centering
\small
\caption{
Ablation and semantic drift analysis. 
Left: role ablation on Mystery BlocksWorld. 
Right: distribution of semantic drift types among solver-successful but semantically unfaithful outputs.
}
\label{tab:ablation-drift-main}
\begin{minipage}{0.43\textwidth}
\centering
\textbf{(a) Role ablation}\\[2pt]
\resizebox{\linewidth}{!}{
\begin{tabular}{@{}lcc@{}}
\toprule
\textbf{Configuration} 
& \textbf{Success} 
& \(\boldsymbol{\Delta}\) \\
\midrule
Full: Actor + Judge + Editor 
& \textbf{71.0} & -- \\
\midrule
Actor only 
& 0.0 & -71.0 \\
Editor + Judge 
& 31.5 & -39.5 \\
Actor + Judge 
& 20.5 & -50.5 \\
Actor + Editor 
& 0.0 & -71.0 \\
Separate models, \(3\times\)7B 
& 0.0 & -71.0 \\
\bottomrule
\end{tabular}
}
\end{minipage}
\hfill
\begin{minipage}{0.54\textwidth}
\centering
\textbf{(b) Drift categories}\\[2pt]
\resizebox{\linewidth}{!}{
\begin{tabular}{@{}lcccc@{}}
\toprule
\textbf{Method} 
& \textbf{Goal weak.} 
& \textbf{Type drift} 
& \textbf{Schema drift} 
& \textbf{Constraint omit.} \\
\midrule
LLM+P 
& 31.2 & 18.8 & 34.5 & 15.5 \\
Solver-only RL 
& 38.4 & 21.7 & 24.6 & 15.3 \\
Ours w/o Judge 
& 34.9 & 19.5 & 27.1 & 18.5 \\
\midrule
\rowcolor{gray!20}
\textbf{Ours} 
& 18.6 & 12.4 & 47.2 & 21.8 \\
\bottomrule
\end{tabular}
}
\end{minipage}
\end{table*}

\begin{figure*}[ht]
    \centering
    \includegraphics[width=1\textwidth]{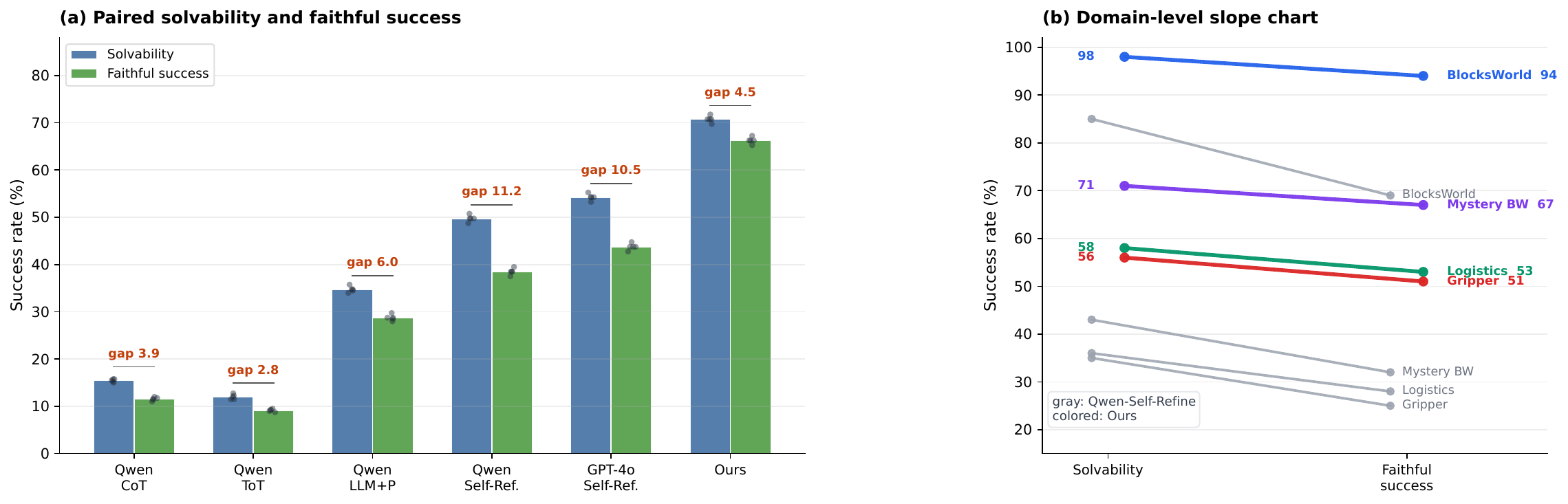}
    \caption{
    \textbf{Solvability--faithfulness gap analysis.}
    Left: paired bars compare solver success and faithful success under matched-budget settings. 
    Right: domain-level slope chart comparing solver success and faithful success.
    }
    \label{fig:solvability-faithfulness-gap-main}
\end{figure*}

\begin{figure*}[ht]
    \centering
    \includegraphics[width=1\textwidth]{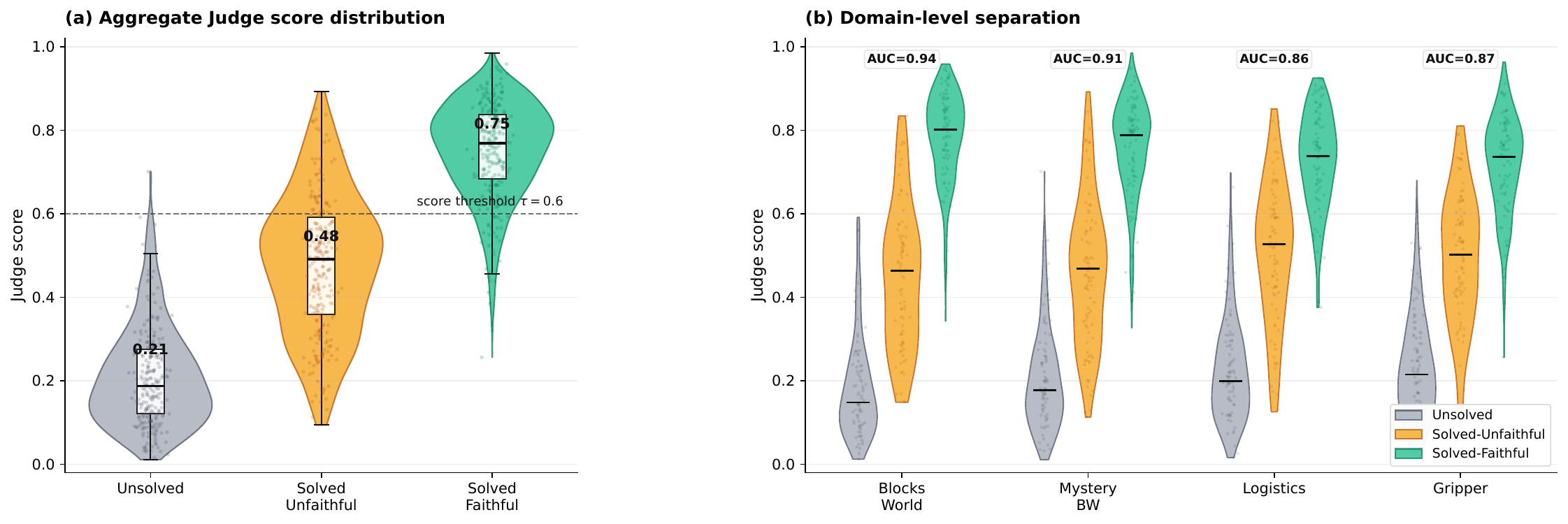}
    \caption{
    \textbf{Judge score distribution over post-hoc semantic categories.}
    Left: aggregate violin plot over all PlanBench domains. 
    Right: domain-level grouped violin plot.
    }
    \label{fig:judge-drift-violin-main}
\end{figure*}

\begin{table}[ht]
\centering
\small
\caption{
Post-hoc Judge separation analysis. 
Scores are reported as mean \(\pm\) standard deviation. 
AUC measures how well the Judge separates faithful outputs from solved-but-unfaithful outputs among solver-successful candidates.
}
\label{tab:judge-drift-main}
\resizebox{0.75\linewidth}{!}{
\begin{tabular}{@{}lccccc@{}}
\toprule
\textbf{Domain}
& \textbf{\#Cand.}
& \textbf{Unsolved}
& \textbf{Solved-Unfaithful}
& \textbf{Solved-Faithful}
& \textbf{AUC} \\
\midrule
BlocksWorld
& 700
& \(0.16\pm0.09\)
& \(0.42\pm0.13\)
& \(0.80\pm0.10\)
& 0.94 \\
Mystery BW
& 830
& \(0.18\pm0.10\)
& \(0.47\pm0.14\)
& \(0.77\pm0.11\)
& 0.90 \\
Logistics
& 800
& \(0.20\pm0.11\)
& \(0.50\pm0.15\)
& \(0.74\pm0.12\)
& 0.87 \\
Gripper
& 780
& \(0.21\pm0.11\)
& \(0.51\pm0.15\)
& \(0.73\pm0.13\)
& 0.85 \\
\midrule
\textbf{All}
& 3110
& \(0.19\pm0.11\)
& \(0.48\pm0.15\)
& \(0.76\pm0.12\)
& 0.89 \\
\bottomrule
\end{tabular}
}
\end{table}

\textbf{Ablation and semantic drift types.}
Table~\ref{tab:ablation-drift-main} shows two diagnostic analyses. 
The ablation results on Mystery BlocksWorld indicate that Actor, Judge, and Editor are all necessary: removing any role causes a large drop, and using three separate \(7\)B models collapses despite larger total parameter count. 
The drift-type analysis shows that our method reduces direct shortcut-like failures such as goal weakening and object/type drift; the remaining drift cases are dominated by harder action-schema errors. 
Detailed interpretation is deferred to Appendix~\ref{app:ablation-drift-discussion}.

\textbf{Repair dynamics and diagnostic evolution.}
Figure~\ref{fig:repair-diagnostic-main} shows how success accumulates across Editor repair steps. 
The initial Actor output solves \(46.2\%\) of instances on average, and bounded repair increases success to \(70.8\%\). 
Most improvement occurs early: the first two repairs contribute \(17.3\) out of the total \(24.6\) point gain, indicating that the Editor effectively corrects many localized defects rather than relying on long iterative search. 
The diagnostic heatmap further shows that syntax and typing errors are removed first, while remaining failures increasingly concentrate in harder schema and transition-model errors. 
This pattern suggests a natural repair hierarchy: shallow formalization errors are rapidly resolved, whereas residual failures require deeper corrections to action semantics and state transitions. 
The full repair protocol is described in Appendix~\ref{app:repair-protocol}.

\textbf{Solvability--faithfulness gap.}
Figure~\ref{fig:solvability-faithfulness-gap-main} visualizes the gap between solver success and faithful success. 
Our method has a \(4.5\)-point gap, compared with \(12.3\) points for Qwen-Self-Refine+Solver and \(11.3\) points for GPT-4o-Self-Refine+Solver. 
Thus, stronger prompting or stronger base models can improve solvability, but they do not fully prevent semantic drift. 
More details are provided in Appendix~\ref{app:gap-discussion}.

\textbf{Judge separation of faithful and unfaithful outputs.}
Finally, Figure~\ref{fig:judge-drift-violin-main} and Table~\ref{tab:judge-drift-main} test whether the Judge distinguishes solver-successful but semantically unfaithful outputs from truly faithful ones. 
The Judge assigns low scores to unsolved outputs, intermediate scores to solved-but-unfaithful outputs, and high scores to faithful outputs, with an overall AUC of \(0.89\) among solver-successful candidates. 
This suggests that the Judge captures structural quality signals beyond raw solver executability. 
Candidate construction and category definitions are given in Appendix~\ref{app:judge-protocol}.

\section{Conclusion}
We studied annotation-free natural-language-to-PDDL planning, where a model must learn executable and faithful symbolic specifications from solver feedback alone. 
We introduced a solver-grounded multi-role reinforcement learning framework that conditions a single language model as an Actor, Judge, and Editor, separating global generation, calibrated verification, and diagnostic-conditioned repair. 
Across PlanBench and zero-shot transfer benchmarks, the results show that structuring solver feedback into complementary roles improves both planning success and semantic faithfulness, rather than merely increasing solver executability. 
These findings suggest that verifiable environments can provide scalable supervision for neuro-symbolic planning when feedback is converted into generation, verification, and repair signals. 

A natural next step is to extend this paradigm from classical PDDL domains to richer interactive environments where symbolic constraints, tool feedback, and natural-language goals co-evolve.

\bibliographystyle{plainnat}
\bibliography{reference}

\clearpage

\appendix

\section{Reproducibility Details}
\label{app:repro}

\paragraph{Scope.}
This appendix consolidates all implementation- and experiment-related details required to reproduce the proposed
solver-grounded multi-role reinforcement learning (RL) framework for NL-to-PDDL planning.

\subsection{Task, Data, and Evaluation}
\paragraph{Task definition.}
Given a natural-language planning description $x$, the goal is to generate a PDDL specification
$y=(\mathcal{D},\mathcal{P})$ consisting of a domain file $\mathcal{D}$ and a problem file $\mathcal{P}$ that can be
verified by a PDDL solver (i.e., achieves the goal with valid syntax).

\paragraph{Metric.}
Primary metric is planning success rate (\%), measured by the solver's binary verification result.

\subsection{Grounded Verifier (PDDL Solver)}
\paragraph{Solver.}
The environment verifier $E$ is the Fast Downward planner with a 60-second timeout.
Given a candidate specification $y$, the solver returns $(r,f)=E(y)$, where
$r\in\{0,1\}$ indicates goal achievement, and $f$ contains structured diagnostics such as:
syntax errors, unsatisfied preconditions, unreachable goals, and execution traces.

\subsection{Model, Roles, and Parameter Sharing}
\paragraph{Base language model.}
\textbf{Qwen2.5-7B} is used as the pretrained backbone (initialized from pretrained weights only; no supervised fine-tuning).

\paragraph{One Brain, Three Roles.}
A single model plays three roles via learned role conditioning:
Actor ($r=A$), Judge ($r=J$), and Editor ($r=E$).
Role conditioning is implemented with a learnable role embedding $e_r$:
\[
p_\theta(y \mid x, r)=\mathrm{LM}_\theta\!\left(y \mid [e_r; x]\right).
\]

\paragraph{Parameter sharing.}
Parameters are partitioned as $\theta=\theta_{\mathrm{shared}}\cup\theta_A\cup\theta_J\cup\theta_E$,
where $\lvert\theta_{\mathrm{shared}}\rvert\approx 0.95\lvert\theta\rvert$.
Role-specific parameters (about $1.67\%$ each) are lightweight projection heads.

\subsection{RL Objective, Rewards, and Decoding}
\paragraph{MDP objective.}
The generation-refinement process is modeled as an MDP with objective:
\[
J(\theta)=\mathbb{E}_{\tau\sim \pi_\theta}\left[\sum_{t=0}^{T}\gamma^t R(s_t,a_t)\right],
\]
where $\gamma$ is the discount factor = 0.9.

\paragraph{Actor (generation).}
The Actor generates an initial PDDL specification $y_0$ from the natural-language input.
During training, nucleus sampling with temperature $T_{\mathrm{samp}}=0.7$ is used,
and switches to greedy decoding at inference.

\paragraph{Actor reward.}
The Actor reward combines solver success and the Judge's quality score:
\[
R_A = r_{\text{solver}} + \lambda_J\cdot s_{\text{Judge}},
\]
where $r_{\text{solver}}\in\{0,1\}$ and $s_{\text{Judge}}\in[0,1]$. $\lambda_J=0.3$.

\paragraph{Judge (quality prediction).}
The Judge predicts a quality score:
\[
s_{\text{Judge}}=\sigma(f_\theta(x,y_0,r=J)).
\]
Judge reward enforces consistency with the solver outcome using threshold $\tau$:
\[
R_J =
\begin{cases}
+1 & s_{\text{Judge}}>\tau \ \wedge\ r_{\text{solver}}=1,\\
-1 & s_{\text{Judge}}>\tau \ \wedge\ r_{\text{solver}}=0,\\
+0.5 & s_{\text{Judge}}\le\tau \ \wedge\ r_{\text{solver}}=0,\\
-0.5 & s_{\text{Judge}}\le\tau \ \wedge\ r_{\text{solver}}=1.
\end{cases}
\]
with $\tau=0.5$.

\paragraph{Editor (bounded refinement).}
When the solver fails ($r_0=0$), the Editor iteratively refines the specification using solver diagnostics $f_t$:
\[
y_{t+1} \sim \pi_E(\cdot \mid x,y_t,f_t),
\quad t=0,1,\dots,T.
\]
The maximum refinement steps is $T_{\max}=5$.

\paragraph{Editor reward.}
The Editor trades off final success and number of edits:
\[
R_E = r_{\text{solver}}(y_T) - \beta\cdot T,
\]
where $\beta=0.1$ and $T\le T_{\max}$.

\subsection{Training Protocol and Optimization}
\paragraph{Training loop.}
Training is conducted in cycles. Each cycle includes:
(1) Actor generates $y_0$ and queries the solver for $(r_0,f_0)$;
(2) Judge scores $s_{\text{Judge}}$ and receives $R_J$ based on solver outcome;
(3) if $r_0=0$, Editor refines up to $T_{\max}$ steps and receives $R_E$ based on final outcome.

\paragraph{Joint parameter update.}
All roles share the backbone and are updated with an aggregate gradient:
\[
\nabla_\theta J = \nabla_\theta J_A + \lambda_J \nabla_\theta J_J + \lambda_E \nabla_\theta J_E,
\]
where $\lambda_J=0.3$ and $\lambda_E=0.5$.

\paragraph{RL algorithm (PPO).}
Optimization uses Proximal Policy Optimization (PPO) with:
learning rate $3\times 10^{-5}$,
batch size $32$,
and clipping ratio $\epsilon=0.2$.



\subsection{Hyperparameter Summary}
Shown in Table \ref{tab:repro_hparams}
\begin{table}[t]
\centering
\small
\begin{tabular}{ll}
\hline
\textbf{Component} & \textbf{Setting} \\
\hline
Base model & Qwen2.5-7B \\
Environment & PlanBench \\
Solver & Fast Downward, 60s timeout \\
Learning rate & $3\times 10^{-5}$ \\
Batch size & 32 \\
PPO clip & $\epsilon=0.2$ \\
Actor decoding (train) & nucleus sampling, temperature $T_{\mathrm{samp}}=0.7$ \\
Actor decoding (test) & greedy \\
Actor reward weight & $\lambda_J=0.3$ \\
Judge threshold & $\tau=0.5$ \\
Editor max steps & $T_{\max}=5$ \\
Editor step penalty & $\beta=0.1$ \\
Role objective weight & $\lambda_E=0.5$ \\
Discount factor & $\gamma$ = 0.9 \\
\hline
\end{tabular}
\caption{Hyperparameters and settings.}
\label{tab:repro_hparams}
\end{table}

\subsection{Compute Resources}
\label{app:compute-resources}

All experiments were conducted on a single multi-GPU compute node with \(8\times\) NVIDIA A100 GPUs. 
Each GPU has 80GB memory, and the node has 64 CPU cores and 512GB system memory. 
The Qwen2.5-7B backbone was trained with mixed precision. 
Fast Downward was executed on CPU with a fixed 60-second timeout for each generated PDDL specification.

The main solver-grounded multi-role RL run uses approximately 12{,}000 solver interactions across 8 training cycles. 
A full training run takes approximately 12--16 hours on the above hardware. 

\section{Additional Method Details}
\label{app:method-details}

\subsection{Solver Environment and Diagnostics}
\label{app:solver-details}

The solver environment \(\mathcal{E}\) is implemented with Fast Downward using a 60-second timeout. Given a generated PDDL specification \(y=(y^{\mathrm{dom}},y^{\mathrm{prob}})\), the solver first checks whether the domain and problem files can be parsed and grounded. If parsing and grounding succeed, the planner searches for a plan that reaches the stated goal. The returned label \(b\in\{0,1\}\) is defined as \(b=1\) when the solver returns a valid plan within the timeout and \(b=0\) otherwise.

The diagnostic field \(d\) records the most informative failure mode available from the solver pipeline. We use diagnostics from four categories: syntax and parse errors, type and object mismatches, invalid operator definitions such as missing preconditions or effects, and search-level failures such as unreachable goals. The Editor receives the raw diagnostic text together with the current specification. During training, generated specifications and solver outputs are stored as on-policy interaction data for Actor, Judge, and Editor updates.

\subsection{Optimization and Hyperparameters}
\label{app:training-details}

\paragraph{Decoding.}
During training, the Actor and Editor use stochastic decoding to maintain exploration. We use nucleus sampling with temperature \(\tau_{\mathrm{temp}}=0.7\). During evaluation, both roles use greedy decoding for deterministic comparison. The repair horizon is fixed to \(H_{\max}=5\). If the initial Actor output succeeds, no Editor step is taken.

\paragraph{Actor and Editor PPO objectives.}
For role \(\rho\in\{\mathrm{A},\mathrm{E}\}\), let \(o_t^\rho\) denote the context given to the role policy, \(a_t^\rho\) the generated token sequence, and \(\hat{A}_t^\rho\) the advantage computed from the corresponding role reward. PPO minimizes
\begin{equation}
  \mathcal{L}_{\mathrm{PPO}}^{\rho}
  =
  -
  \mathbb{E}_{t}
  \left[
    \min
    \left(
      q_t^\rho(\theta)\hat{A}_t^\rho,
      \operatorname{clip}\!\left(q_t^\rho(\theta),1-\epsilon_{\mathrm{ppo}},1+\epsilon_{\mathrm{ppo}}\right)\hat{A}_t^\rho
    \right)
  \right],
  \label{eq:ppo-appendix}
\end{equation}
where
\[
  q_t^\rho(\theta)
  =
  \frac{
    \pi_\theta(a_t^\rho\mid o_t^\rho,\rho)
  }{
    \pi_{\theta_{\mathrm{old}}}(a_t^\rho\mid o_t^\rho,\rho)
  } .
\]
The Actor advantage is computed from \(R_{\mathrm{A}}\), and the Editor advantage is computed from \(R_{\mathrm{E}}\). We use \(\epsilon_{\mathrm{ppo}}=0.2\).

\paragraph{Judge loss.}
The Judge is trained as a binary predictor of solver success. The binary cross-entropy term is
\[
  \mathrm{BCE}(s,b)
  =
  -b\log s-(1-b)\log(1-s).
\]
The banded penalty is
\begin{equation}
  \ell_{\mathrm{band}}(s,b)
  =
  (1-b)\max(0,s-\tau_{\mathrm{J}})^2
  +\frac{1}{2}b\max(0,\tau_{\mathrm{J}}-s)^2 .
  \label{eq:band-loss}
\end{equation}
The first term penalizes false positives above the decision threshold, while the second term penalizes under-confident scores on solver-successful specifications. We use \(\tau_{\mathrm{J}}=0.5\).

\paragraph{Joint update.}
All trainable parameters are updated with
\[
  \mathcal{L}_{\mathrm{total}}
  =
  \mathcal{L}_{\mathrm{PPO}}^{\mathrm{A}}
  +\lambda_{\mathrm{E}}\mathcal{L}_{\mathrm{PPO}}^{\mathrm{E}}
  +\lambda_{\mathrm{J}}\mathcal{L}_{\mathrm{J}} .
\]
The shared backbone receives gradients from all three roles. The Actor head receives Actor gradients, the Judge head receives Judge gradients, and the Editor head receives Editor gradients.

\begin{table}[t]
\centering
\small
\caption{Hyperparameters used in the solver-grounded multi-role training procedure.}
\label{tab:method-hparams}
\begin{tabular}{ll}
\toprule
\textbf{Component} & \textbf{Setting} \\
\midrule
Base model & Qwen2.5-7B \\
Solver & Fast Downward \\
Solver timeout & 60 seconds \\
Actor decoding during training & nucleus sampling, \(\tau_{\mathrm{temp}}=0.7\) \\
Actor decoding during evaluation & greedy decoding \\
Editor decoding during training & nucleus sampling, \(\tau_{\mathrm{temp}}=0.7\) \\
Editor decoding during evaluation & greedy decoding \\
Maximum repair steps & \(H_{\max}=5\) \\
PPO clip coefficient & \(\epsilon_{\mathrm{ppo}}=0.2\) \\
Learning rate & \(3\times 10^{-5}\) \\
Batch size & 32 \\
Judge threshold & \(\tau_{\mathrm{J}}=0.5\) \\
Actor shaping weight & \(\lambda_{\mathrm{J}}=0.3\) \\
Editor loss weight & \(\lambda_{\mathrm{E}}=0.5\) \\
Editor step penalty & \(\beta=0.1\) \\
\bottomrule
\end{tabular}
\end{table}

\section{Judge Calibration Analysis}
\label{app:judge-analysis}

This section proves the calibration statement used in Section~\ref{sec:judge-calibration}. The result is distributional and on-policy: it concerns the generated specifications encountered under the current policy.

\paragraph{Setup.}
Let \(X\sim\mathcal{P}_{\mathrm{data}}\) be a task and let \(Y\sim\pi_\theta(\cdot\mid X)\) be a generated specification. The solver-success label is
\(B=b(X,Y)\in\{0,1\}\), and the task-level correctness label is
\(C=c(X,Y)\in\{0,1\}\). The solver and the reference correctness checker are deterministic for a fixed pair \((x,y)\); randomness comes from task sampling and stochastic generation.

Define the Bayes-optimal solver-success predictor
\[
  s_J^*(x,y)=\Pr(B=1\mid X=x,Y=y).
\]
The learned Judge has calibration error
\begin{equation}
  \delta_J
  =
  \sup_{x,y}
  \left|s_J(x,y)-s_J^*(x,y)\right|.
  \label{eq:judge-delta-app}
\end{equation}

\begin{assumption}[Bounded solver--correctness disagreement]
\label{ass:bounded-disagreement}
For the policy-induced distribution over \((X,Y)\), there exists \(\varepsilon<1/2\) such that
\begin{equation}
  \Pr(B\neq C)\le \varepsilon .
  \label{eq:bounded-disagreement}
\end{equation}
\end{assumption}

\begin{lemma}[Bayes-optimality of the Judge target]
\label{lem:judge-bayes-app}
The minimizer of the conditional cross-entropy loss for predicting \(B\) is
\[
  s_J^*(x,y)=\Pr(B=1\mid X=x,Y=y).
\]
\end{lemma}

\begin{proof}
For a fixed pair \((x,y)\), let \(\eta=\Pr(B=1\mid X=x,Y=y)\). The conditional cross-entropy is
\[
  \ell(s;\eta)=-\eta\log s-(1-\eta)\log(1-s), \qquad s\in(0,1).
\]
Its derivative is \(-\eta/s+(1-\eta)/(1-s)\), which vanishes only at \(s=\eta\). The second derivative is positive on \((0,1)\), so \(s=\eta\) is the unique minimizer.
\end{proof}

\begin{theorem}[Judge--correctness connection]
\label{thm:judge-correctness-app}
Under Assumption~\ref{ass:bounded-disagreement}, any Judge with calibration error \(\delta_J\) satisfies
\begin{equation}
  \mathbb{E}_{(X,Y)}
  \left[
    \left|s_J(X,Y)-C(X,Y)\right|
  \right]
  \le
  \varepsilon+\delta_J .
  \label{eq:judge-L1-app}
\end{equation}
Moreover,
\begin{equation}
  \Pr\!\left(
    \left|s_J(X,Y)-C(X,Y)\right|>\delta_J
  \right)
  \le
  \varepsilon .
  \label{eq:judge-hp-app}
\end{equation}
\end{theorem}

\begin{proof}
By Lemma~\ref{lem:judge-bayes-app}, \(s_J^*(X,Y)=\Pr(B=1\mid X,Y)\). Since \(B\) is deterministic given \((X,Y)\), we have \(s_J^*(X,Y)=B(X,Y)\) almost surely. Therefore,
\[
  |s_J(X,Y)-C(X,Y)|
  \le
  |s_J(X,Y)-s_J^*(X,Y)|
  +
  |B(X,Y)-C(X,Y)|.
\]
Taking expectations yields Eq.~\eqref{eq:judge-L1-app}. For Eq.~\eqref{eq:judge-hp-app}, note that on the event \(B=C\), the first term is at most \(\delta_J\). Thus the event \(|s_J-C|>\delta_J\) can occur only when \(B\neq C\), whose probability is at most \(\varepsilon\).
\end{proof}

\begin{theorem}[Margin ranking bound]
\label{thm:judge-ranking-app}
Let \(Y_1,Y_2\overset{\mathrm{i.i.d.}}{\sim}\pi_\theta(\cdot\mid X)\). Under Assumption~\ref{ass:bounded-disagreement},
\begin{equation}
\begin{aligned}
  \Pr\Big(
    &s_J(X,Y_1)\ge s_J(X,Y_2)+2\delta_J \\
    &\wedge\ C(X,Y_1)<C(X,Y_2)
  \Big)
  \le
  2\varepsilon .
\end{aligned}
\label{eq:ranking-bound-app}
\end{equation}
\end{theorem}

\begin{proof}
Let \(G\) be the event that \(B(X,Y_i)=C(X,Y_i)\) for both \(i=1,2\). By a union bound, \(\Pr(G^c)\le 2\varepsilon\). On \(G\), the learned score differs from the binary correctness value by at most \(\delta_J\) for each candidate. Therefore, a score margin of \(2\delta_J\) cannot rank an incorrect candidate above a correct candidate. The mis-ranking event in Eq.~\eqref{eq:ranking-bound-app} is contained in \(G^c\), so its probability is at most \(2\varepsilon\).
\end{proof}

\subsection{Estimating the Solver--Correctness Disagreement Rate}
\label{app:epsilon-estimation}

For a fixed checkpoint \(\theta\), the disagreement rate can be estimated by sampling tasks \(X_i\), generating specifications \(Y_i\), recording solver labels \(B_i\), and evaluating task-level correctness labels \(C_i\) using a reference validator or audit protocol:
\begin{equation}
  \widehat{\varepsilon}_{\theta}
  =
  \frac{1}{n}
  \sum_{i=1}^{n}
  \mathbf{1}\{B_i\neq C_i\}.
  \label{eq:epsilon-hat-app}
\end{equation}
For any \(\delta\in(0,1)\), Hoeffding's inequality gives
\begin{equation}
  \varepsilon_{\theta}
  \le
  \widehat{\varepsilon}_{\theta}
  +
  \sqrt{\frac{\log(1/\delta)}{2n}}
  \label{eq:epsilon-hoeffding-app}
\end{equation}
with probability at least \(1-\delta\). For a finite set of checkpoints \(\Theta_{\mathrm{ckpt}}\), a uniform bound is obtained by replacing \(\delta\) with \(\delta/|\Theta_{\mathrm{ckpt}}|\).

\section{Why Separate Actor and Editor Roles}
\label{app:actor-editor-analysis}

This section clarifies why the Actor and Editor are trained as separate roles rather than a single policy optimized only for final solver success.

\paragraph{Initial correctness.}
For a policy that produces an initial specification \(Y_0\), define
\[
  C_{\mathrm{init}}(\pi)
  =
  \mathbb{E}\left[c(X,Y_0)\right].
\]
Here \(c(X,Y_0)=1\) means that the initial specification is correct with respect to the task-level reference semantics.

\paragraph{Actor-only alignment.}
If the Actor is trained with the initial solver label \(b(X,Y_0)\), then under the bounded disagreement condition \(\Pr(b(X,Y_0)\neq c(X,Y_0))\le\varepsilon\),
\[
  \left|
  \mathbb{E}[b(X,Y_0)]-C_{\mathrm{init}}(\pi)
  \right|
  \le
  \varepsilon .
\]
Thus the initial solver label remains a distributional surrogate for initial correctness whenever solver success and task-level correctness agree on most on-policy samples.

\paragraph{Final-only Editor objective.}
Consider a monolithic repair policy that both initializes and edits, and receives only the final reward
\[
  R_{\mathrm{mono}}=b(X,Y_T)-\beta T .
\]
Suppose there exists a set of repairable initial specifications \(\mathcal{Y}_{\mathrm{rep}}(x)\) such that every \(y_0\in\mathcal{Y}_{\mathrm{rep}}(x)\) can be repaired to solver success in the same number of steps. Then all initializers supported on \(\mathcal{Y}_{\mathrm{rep}}(x)\) receive the same final return, even if their initial correctness differs. The final-only objective therefore does not identify the quality of \(Y_0\).

\begin{proposition}[Initial specification is unidentifiable under saturated repair]
\label{prop:editor-unidentifiable}
Assume that for every \(x\), all \(y_0\in\mathcal{Y}_{\mathrm{rep}}(x)\) can be repaired to \(b=1\) in exactly \(T^\dagger\) steps. If \(\mathcal{Y}_{\mathrm{rep}}(x)\) contains both correct and incorrect initial specifications, then for any \(\alpha\in[0,1]\) there exists an initializer with \(C_{\mathrm{init}}=\alpha\) and final return \(1-\beta T^\dagger\).
\end{proposition}

\begin{proof}
For each \(x\), choose \(y^+(x),y^-(x)\in\mathcal{Y}_{\mathrm{rep}}(x)\) such that \(c(x,y^+(x))=1\) and \(c(x,y^-(x))=0\). Define the initializer to output \(y^+(x)\) with probability \(\alpha\) and \(y^-(x)\) with probability \(1-\alpha\). By the saturated repair assumption, both choices obtain final return \(1-\beta T^\dagger\). The expected initial correctness is \(\alpha\).
\end{proof}

This proposition motivates the Actor--Editor decomposition. The Actor receives an initial-stage reward, which keeps global formalization tied to the initial specification. The Editor receives a final-stage repair reward, which specializes it for local diagnostic-conditioned correction.

\section{Proof of Reward-Hacking Scaling}
\label{app:sharing-analysis}

This section proves Theorem~\ref{thm:hacking-scaling-main}.

\paragraph{Local hacking gain.}
Let \(J_{\mathrm{A}}(\theta)\) denote the Actor objective and \(C(\theta)\) denote expected task-level correctness. Around a reference point \(\theta_0\), define
\[
  g_{\mathrm{hack}}(\theta_0)
  =
  \nabla_{\theta}J_{\mathrm{A}}(\theta_0)
  -
  a\nabla_{\theta}C(\theta_0),
  \qquad a>0 .
\]
For a unit-norm perturbation \(\Delta\theta\), the first-order Actor gain not explained by correctness is
\[
  \Delta_{\mathrm{hack}}(\theta_0;\Delta\theta)
  =
  g_{\mathrm{hack}}(\theta_0)^\top \Delta\theta .
\]

\paragraph{Admissible perturbation sets.}
For three separate models, the Actor has its own parameter vector \(\theta_{\mathrm{A}}\in\mathbb{R}^{P_{\mathrm{A}}}\). The admissible perturbation set is
\begin{equation}
  \mathcal{H}_{\mathrm{sep}}
  =
  \left\{
    \Delta\theta_{\mathrm{A}}\in\mathbb{R}^{P_{\mathrm{A}}}
    :
    \|\Delta\theta_{\mathrm{A}}\|_2\le 1
  \right\}.
  \label{eq:Hsep-app}
\end{equation}

For the shared-backbone architecture, write
\[
  \theta=(\theta_{\mathrm{sh}},\theta_{\mathrm{A}},\theta_{\mathrm{J}},\theta_{\mathrm{E}}),
\]
where \(\theta_{\mathrm{sh}}\in\mathbb{R}^{p}\) is the shared backbone and \(\theta_{\mathrm{A}}\in\mathbb{R}^{h}\) is the Actor head. We consider perturbations that are neutral to Judge and Editor on the shared backbone and do not modify the Judge or Editor heads:
\begin{equation}
  \mathcal{H}_{\mathrm{shr}}
  =
  \left\{
    \Delta\theta:
    \|\Delta\theta\|_2\le 1,\,
    \Delta\theta_{\mathrm{J}}=0,\,
    \Delta\theta_{\mathrm{E}}=0,\,
    \Delta\theta_{\mathrm{sh}}\in\mathrm{Null}(\Sigma_{\mathrm{JE}})
  \right\}.
  \label{eq:Hshr-app}
\end{equation}
Here \(\Sigma_{\mathrm{JE}}\) is the Judge--Editor backbone sensitivity matrix defined below.

\begin{assumption}[Judge and Editor constrain the shared backbone]
\label{ass:shared-backbone-constraint}
Let \(g_{\mathrm{J}}\in\mathbb{R}^{p}\) and \(g_{\mathrm{E}}\in\mathbb{R}^{p}\) denote stochastic backbone gradients from the Judge and Editor objectives. Define
\begin{equation}
  \Sigma_{\mathrm{JE}}
  =
  \mathbb{E}[g_{\mathrm{J}}g_{\mathrm{J}}^\top]
  +
  \mathbb{E}[g_{\mathrm{E}}g_{\mathrm{E}}^\top].
  \label{eq:SigmaJE-app}
\end{equation}
Assume \(\mathrm{Null}(\Sigma_{\mathrm{JE}})=\{0\}\). Therefore, any Judge- and Editor-neutral perturbation in \(\mathcal{H}_{\mathrm{shr}}\) must lie in the Actor head subspace.
\end{assumption}

\begin{assumption}[Local isotropic hacking-gradient model]
\label{ass:isotropic-hacking}
There exists \(\sigma>0\) such that the restriction of \(g_{\mathrm{hack}}(\theta_0)\) to the relevant admissible subspace is isotropic Gaussian. In the separate setting,
\[
  g_{\mathrm{hack}}(\theta_0)
  \sim
  \mathcal{N}(0,\sigma^2 I_{P_{\mathrm{A}}}).
\]
In the shared setting,
\[
  \mathrm{Proj}_{\mathrm{A}}g_{\mathrm{hack}}(\theta_0)
  \sim
  \mathcal{N}(0,\sigma^2 I_h),
\]
where \(\mathrm{Proj}_{\mathrm{A}}\) denotes projection onto the Actor head subspace.
\end{assumption}

\begin{lemma}[Gaussian supremum over a unit ball]
\label{lem:gaussian-sup-app}
Let \(g\sim\mathcal{N}(0,\sigma^2 I_d)\). Then
\begin{equation}
  \mathbb{E}
  \left[
    \sup_{\|u\|_2\le 1} g^\top u
  \right]
  =
  \mathbb{E}\|g\|_2
  =
  \Theta(\sigma\sqrt{d}).
  \label{eq:gaussian-sup-app}
\end{equation}
\end{lemma}

\begin{proof}
For any fixed \(g\), Cauchy--Schwarz gives
\[
  \sup_{\|u\|_2\le 1} g^\top u=\|g\|_2,
\]
with equality at \(u=g/\|g\|_2\) when \(g\neq 0\). Since \(g\sim\mathcal{N}(0,\sigma^2 I_d)\), \(\|g\|_2/\sigma\) follows a \(\chi_d\) distribution, whose expectation is \(\Theta(\sqrt{d})\).
\end{proof}

\begin{proof}[Proof of Theorem~\ref{thm:hacking-scaling-main}]
For the separate-model setting, \(\mathcal{H}_{\mathrm{sep}}\) is the unit ball in the Actor parameter space. Therefore,
\[
  \sup_{\Delta\theta\in\mathcal{H}_{\mathrm{sep}}}
  \Delta_{\mathrm{hack}}(\theta_0;\Delta\theta)
  =
  \sup_{\|\Delta\theta_{\mathrm{A}}\|_2\le 1}
  g_{\mathrm{hack}}(\theta_0)^\top\Delta\theta_{\mathrm{A}}
  =
  \|g_{\mathrm{hack}}(\theta_0)\|_2 .
\]
By Assumption~\ref{ass:isotropic-hacking} and Lemma~\ref{lem:gaussian-sup-app} with \(d=P_{\mathrm{A}}\),
\[
  \mathbb{E}
  \left[
    \sup_{\Delta\theta\in\mathcal{H}_{\mathrm{sep}}}
    \Delta_{\mathrm{hack}}(\theta_0;\Delta\theta)
  \right]
  =
  \Theta(\sigma\sqrt{P_{\mathrm{A}}}).
\]

For the shared-backbone setting, Assumption~\ref{ass:shared-backbone-constraint} gives
\(\Delta\theta_{\mathrm{sh}}=0\) for every \(\Delta\theta\in\mathcal{H}_{\mathrm{shr}}\), and by construction \(\Delta\theta_{\mathrm{J}}=\Delta\theta_{\mathrm{E}}=0\). Hence every admissible perturbation lies in the Actor head subspace:
\[
  \Delta\theta=(0,\Delta\theta_{\mathrm{A}},0,0),
  \qquad
  \|\Delta\theta_{\mathrm{A}}\|_2\le 1 .
\]
Thus,
\[
  \sup_{\Delta\theta\in\mathcal{H}_{\mathrm{shr}}}
  \Delta_{\mathrm{hack}}(\theta_0;\Delta\theta)
  =
  \sup_{\|\Delta\theta_{\mathrm{A}}\|_2\le 1}
  \left(\mathrm{Proj}_{\mathrm{A}}g_{\mathrm{hack}}(\theta_0)\right)^\top
  \Delta\theta_{\mathrm{A}} .
\]
By Assumption~\ref{ass:isotropic-hacking} and Lemma~\ref{lem:gaussian-sup-app} with \(d=h\),
\[
  \mathbb{E}
  \left[
    \sup_{\Delta\theta\in\mathcal{H}_{\mathrm{shr}}}
    \Delta_{\mathrm{hack}}(\theta_0;\Delta\theta)
  \right]
  =
  \mathbb{E}
  \left[
    \left\|\mathrm{Proj}_{\mathrm{A}}g_{\mathrm{hack}}(\theta_0)\right\|_2
  \right]
  \lesssim
  \sigma\sqrt{h}.
\]
This proves both claims.
\end{proof}

\paragraph{Effective-rank variant.}
If \(\mathrm{Null}(\Sigma_{\mathrm{JE}})\) has dimension \(k>0\), the same proof gives a shared-architecture scaling of order
\[
  \Theta(\sigma\sqrt{h+k}),
\]
where \(k\) is the number of Judge- and Editor-neutral backbone directions. Thus the relevant quantity is the dimension of the role-private or role-neutral subspace, rather than the total backbone parameter count.

\section{Proofs for One Brain Three Roles}
\label{app:theory}

Throughout, $X$ denotes a task sampled from $\mathcal{D}$, $Y$ a PDDL specification produced by the policy, and $R,C \in \{0,1\}$ are respectively the \emph{solver success indicator} (``solvability'') and the \emph{ground-truth plan correctness indicator} (``correctness'') of the final plan for $(X,Y)$: $R=1$ iff the PDDL solver parses $Y$ and returns a plan that reaches its goal under the semantics encoded by $(X,Y)$, and $C=1$ iff that final plan is correct under the true task semantics.
We write $C(\theta) = \mathbb{E}[C]$ for the expected correctness under policy $\pi_\theta$.

\subsection{Judge--Correctness Consistency}
\label{app:judge-consistency}

\paragraph{Deterministic oracles; randomness only from sampling.}
Although the planner and the validator are deterministic maps for any fixed
natural-language task $x$ and PDDL specification $y$, the pair $(X,Y)$ is random
because (i) tasks are sampled as $X\sim\mathcal{D}$ and (ii) specifications are
sampled as $Y\sim\pi_\theta(\cdot\mid X)$ (e.g., via stochastic decoding).
Hence, $R$ and $C$ are random variables through $(X,Y)$, even though
their dependence on $(x,y)$ is deterministic.

Formally, there exist measurable functions
$r,c:\mathcal{X}\times\mathcal{Y}\to\{0,1\}$ such that
\begin{equation}
  R \;=\; r(X,Y)\in\{0,1\},
  \qquad
  C \;=\; c(X,Y)\in\{0,1\}.
  \label{eq:det-RC}
\end{equation}
We define the conditional success and correctness probabilities
\begin{equation}
\eta(x,y) \coloneqq \mathbb{E}[R\mid X=x,Y=y] = \Pr(R=1\mid X=x,Y=y),
\end{equation}
\begin{equation}
\kappa(x,y) \coloneqq \mathbb{E}[C\mid X=x,Y=y] = \Pr(C=1\mid X=x,Y=y).
\label{eq:eta-kappa}
\end{equation}

Under determinism, $\eta(x,y)=r(x,y)$ and $\kappa(x,y)=c(x,y)$ are $\{0,1\}$-valued.
We keep the probabilistic notation because all bounds below are taken with respect
to the \emph{distribution of} $(X,Y)$ induced by $(\mathcal{D},\pi_\theta)$.

\paragraph{Bounded solver--correctness drift (on-policy).}
We allow ``specification drift'' in which the planner declares success for a
generated specification while the resulting behavior is not correct under the
ground-truth validator; we assume such drift is rare on the distribution induced
by policies of interest.

\begin{assumption}[Bounded solver--correctness drift rate.]
\label{ass:bounded-noise}
There exists $\varepsilon\in[0,1/2)$ such that for every policy $\pi_\theta$
considered in the analysis,
\begin{equation}
  \varepsilon_\theta
  \;\coloneqq\;
  \Pr_{X\sim\mathcal{D},\,Y\sim\pi_\theta(\cdot\mid X)}\!\bigl(R\neq C\bigr)
  \;\le\; \varepsilon.
  \label{eq:drift-rate}
\end{equation}
\end{assumption}
Justification and empirical results in \ref{app:epsilon-justification}

\paragraph{Judge training and calibration.}
The Judge outputs $s_J:\mathcal{X}\times\mathcal{Y}\to(0,1)$ and is trained by
conditional cross-entropy for predicting the planner outcome $R$:
\begin{equation}
  \mathcal{L}(s)
  \;=\;
  \mathbb{E}\!\left[
    - R \log s(X,Y) - (1-R)\log\bigl(1-s(X,Y)\bigr)
  \right],
  \label{eq:judge-ce}
\end{equation}
where the expectation is taken over the joint distribution of $(X,Y)$ used to
train the Judge (typically the on-policy distribution induced by the Actor).
Let $s_J^{*}(x,y)$ denote the Bayes pointwise minimizer of~\eqref{eq:judge-ce}.
We quantify approximation/calibration by
\begin{equation}
  \delta_J
  \;\coloneqq\;
  \sup_{x,y}\bigl|s_J(x,y)-s_J^*(x,y)\bigr|.
  \label{eq:judge-cal-err}
\end{equation}

\begin{lemma}[Bayes-optimal Judge under cross-entropy]
\label{lem:judge-bayes}
For every $(x,y)$, the pointwise minimizer of~\eqref{eq:judge-ce} satisfies
\begin{equation}
  s_J^*(x,y)
  \;=\;
  \eta(x,y)
  \;=\;
  \Pr(R=1\mid X=x,Y=y).
  \label{eq:bayes}
\end{equation}
\end{lemma}

\begin{proof}
Fix $(x,y)$ and write $\eta=\eta(x,y)$.
The conditional cross-entropy for predicting $R$ from a score $s\in(0,1)$ is
$\ell(s;\eta)=-\eta\log s-(1-\eta)\log(1-s)$.
Differentiating gives
$\frac{\partial\ell}{\partial s}(s;\eta)=-\eta/s+(1-\eta)/(1-s)$, which vanishes
only at $s=\eta$.
Moreover,
$\frac{\partial^2\ell}{\partial s^2}(s;\eta)=\eta/s^2+(1-\eta)/(1-s)^2>0$ for all
$s\in(0,1)$, hence $\ell(\cdot;\eta)$ is strictly convex and $s=\eta$ is the
unique minimizer.
\end{proof}

\paragraph{Bounding solvability vs.\ correctness (distributionally).}
The key deterministic fact is that, since $R,C\in\{0,1\}$ and both are deterministic
functions of $(X,Y)$, the pointwise gap $|\eta(X,Y)-\kappa(X,Y)|$ is exactly the
drift indicator $\mathbb{I}\{R\neq C\}$.

\begin{lemma}[Distributional solvability--correctness gap]
\label{lem:eta-kappa-gap}
Under Assumption~\ref{ass:bounded-noise},
\begin{equation}
  \mathbb{E}\bigl[\,|\eta(X,Y)-\kappa(X,Y)|\,\bigr]
  \;=\;
  \Pr(R\neq C)
  \;\le\;
  \varepsilon.
  \label{eq:eta-kappa-L1}
\end{equation}
Consequently,
\begin{equation}
  \bigl|\mathbb{E}[R]-\mathbb{E}[C]\bigr|
  \;\le\;
  \varepsilon.
  \label{eq:RC-gap}
\end{equation}
\end{lemma}

\begin{proof}
Because $R=r(X,Y)$ and $C=c(X,Y)$ are $\{0,1\}$-valued, almost surely
$|R-C|=\mathbb{I}\{R\neq C\}$.
Under determinism, $\eta(X,Y)=\mathbb{E}[R\mid X,Y]=R$ and
$\kappa(X,Y)=\mathbb{E}[C\mid X,Y]=C$ almost surely.
Thus $|\eta(X,Y)-\kappa(X,Y)|=|R-C|=\mathbb{I}\{R\neq C\}$.
Taking expectations yields~\eqref{eq:eta-kappa-L1}, and
\eqref{eq:RC-gap} follows from Jensen:
$|\mathbb{E}[R]-\mathbb{E}[C]|=|\mathbb{E}[R-C]|
\le \mathbb{E}[|R-C|]\le \varepsilon$.
\end{proof}

\begin{theorem}[Judge--correctness consistency]
\label{thm:judge-correctness}
Let $s_J^*$ be the Bayes-optimal Judge from Lemma~\ref{lem:judge-bayes}, and
let $s_J$ be any learned Judge with calibration error $\delta_J$ defined in
\eqref{eq:judge-cal-err}. Under Assumption~\ref{ass:bounded-noise},
\begin{align}
  \mathbb{E}\bigl[\,|s_J(X,Y)-\kappa(X,Y)|\,\bigr]
  &\;\le\;
  \varepsilon+\delta_J,
  \label{eq:judge-L1}\\
  \Pr\bigl(|s_J(X,Y)-\kappa(X,Y)|>\delta_J\bigr)
  &\;\le\;
  \varepsilon.
  \label{eq:judge-hp}
\end{align}
Moreover, if $Y_1,Y_2\overset{\mathrm{i.i.d.}}{\sim}\pi_\theta(\cdot\mid X)$, then
the Judge's \emph{margin} controls correctness ranking up to drift:
\begin{equation}
  \Pr\!\Bigl(
    C(X,Y_1) < C(X,Y_2)
    \ \wedge\
    s_J(X,Y_1) \ge s_J(X,Y_2) + 2\delta_J
  \Bigr)
  \;\le\;
  2\varepsilon.
  \label{eq:judge-ranking}
\end{equation}
\end{theorem}

\begin{proof}
By Lemma~\ref{lem:judge-bayes}, $s_J^*(x,y)=\eta(x,y)$.

\emph{(i) Expected consistency.}
By triangle inequality,
$|s_J-\kappa|\le |s_J-s_J^*|+|s_J^*-\kappa|=|s_J-s_J^*|+|\eta-\kappa|$.
Taking expectations and using $\mathbb{E}[|s_J-s_J^*|]\le\delta_J$ and
Lemma~\ref{lem:eta-kappa-gap} gives~\eqref{eq:judge-L1}.

\emph{(ii) High-probability consistency.}
On the event $\{R=C\}$ we have $\kappa=\eta=s_J^*$, hence
$|s_J-\kappa|=|s_J-s_J^*|\le\delta_J$.Therefore
$\{|s_J-\kappa|>\delta_J\}\subseteq\{R\neq C\}$ and
$\Pr(|s_J-\kappa|>\delta_J)\le\Pr(R\neq C)\le\varepsilon$, proving
\eqref{eq:judge-hp}.

\emph{(iii) Ranking with a margin.}
Let $G\coloneqq\{R(X,Y_1)=C(X,Y_1)\}\cap\{R(X,Y_2)=C(X,Y_2)\}$.
By a union bound and Assumption~\ref{ass:bounded-noise},
$\Pr(G^c)\le \Pr(R\neq C \text{ for }(X,Y_1))+\Pr(R\neq C \text{ for }(X,Y_2))
\le 2\varepsilon$.
On $G$, we have $C(X,Y_i)=R(X,Y_i)=\eta(X,Y_i)$ for $i=1,2$, and
$|s_J(X,Y_i)-\eta(X,Y_i)|\le\delta_J$.
Hence on $G$,
$s_J(X,Y_1)\ge s_J(X,Y_2)+2\delta_J \Rightarrow \eta(X,Y_1)\ge \eta(X,Y_2)
\Rightarrow C(X,Y_1)\ge C(X,Y_2)$.
Thus the event in~\eqref{eq:judge-ranking} can only occur on $G^c$, so its
probability is at most $2\varepsilon$.
\end{proof}

\subsection{Justifying the Drift-Rate Parameter $\varepsilon$}
\label{app:epsilon-justification}

This subsection justifies Assumption~\ref{ass:bounded-noise} (bounded solver--correctness drift rate)
and outlines how to set $\varepsilon$ in a deterministic planning environment.

\paragraph{Deterministic oracles; on-policy drift rate.}
The planner and the (ground-truth) validator are deterministic given a task $x$ and a specification $y$:
\[
  R=r(x,y)\in\{0,1\},\qquad C=c(x,y)\in\{0,1\}.
\]
Randomness enters only through sampling $X\sim\mathcal D$ and stochastic generation
$Y\sim\pi_\theta(\cdot\mid X)$.
Define the drift indicator
\begin{equation}
  D \;\coloneqq\; \mathbb{I}\{R\neq C\}\in\{0,1\}.
  \label{eq:D-def}
\end{equation}
For a fixed policy $\pi_\theta$, we define the induced \emph{on-policy drift rate}
\begin{equation}
  \varepsilon_\theta
  \;\coloneqq\;
  \Pr_{X\sim\mathcal D,\;Y\sim\pi_\theta(\cdot\mid X)}(R\neq C)
  \;=\;
  \mathbb{E}[D]
  \;=\;
  \mathbb{E}\bigl[|R-C|\bigr].
  \label{eq:eps-theta}
\end{equation}
Assumption~\ref{ass:bounded-noise} postulates that for the policy class of interest,
$\sup_{\theta}\varepsilon_\theta \le \varepsilon$ for some finite $\varepsilon<\tfrac12$.
Note that this assumption is \emph{distributional}: it bounds the \emph{frequency} of drift under the
policy-induced distribution, rather than imposing a pointwise (conditional) noise model.

\paragraph{Empirical evidence from solvability vs.\ correctness}
\citeauthor{huang2025limitlanguagemodelsplanning} (\citeyear{huang2025limitlanguagemodelsplanning}) report two evaluation metrics
closely aligned with our notation:
\emph{Solvability} (whether a planner finds a plan) and
\emph{Correctness} (whether the plan is accepted under a higher-fidelity reference semantics/validator).
In their experiments across domains and models (e.g. GPT-40, Llama-8B), the observed gap between the two metrics is typically small (typically well below $0.1$ $<$ $1/2$),
suggesting that solver--correctness drift is limited for realistic LLM policies.

Formally, for any fixed $\pi_\theta$ define the induced on-policy rates
\[
  S(\theta)\coloneqq \Pr(R=1),
  \qquad
  K(\theta)\coloneqq \Pr(C=1),
\]
where probabilities are over $(X,Y)\sim(\mathcal D,\pi_\theta)$.
Then the observable solvability--correctness gap is always controlled by the drift rate:
\begin{equation}
  |S(\theta)-K(\theta)|
  \;=\;
  \bigl|\mathbb{E}[R]-\mathbb{E}[C]\bigr|
  \;\le\;
  \mathbb{E}\bigl[|R-C|\bigr]
  \;=\;
  \varepsilon_\theta.
  \label{eq:gap-vs-drift}
\end{equation}
Moreover, under the common evaluation convention that correctness is defined only for solver-produced plans
(so $C\le R$ almost surely), the gap equals the (false-positive) drift probability:
\begin{equation}
  \varepsilon_\theta
  \;=\;
  \Pr(R=1,C=0)
  \;=\;
  \Pr(R=1)-\Pr(C=1)
  \;=\;
  S(\theta)-K(\theta).
  \label{eq:gap-equals-drift}
\end{equation}
Thus, the small empirical solvability--correctness gaps reported by
\citeauthor{huang2025limitlanguagemodelsplanning} provide direct evidence that $\varepsilon_\theta$
is small in practice for realistic policies, and in particular that taking a uniform constant
$\varepsilon<\tfrac12$ is reasonable.

\paragraph{Estimating $\varepsilon_\theta$ and choosing a uniform $\varepsilon$.}
For a fixed policy $\pi_\theta$, we can estimate $\varepsilon_\theta$ by sampling
$\{X_i\}_{i=1}^n\sim\mathcal D$, generating $Y_i\sim\pi_\theta(\cdot\mid X_i)$, and computing
$R_i=r(X_i,Y_i)$ and $C_i=c(X_i,Y_i)$, yielding the empirical drift rate
\begin{equation}
  \widehat{\varepsilon}_\theta
  \;\coloneqq\;
  \frac{1}{n}\sum_{i=1}^n \mathbb{I}\{R_i\neq C_i\}.
  \label{eq:eps-hat}
\end{equation}
By Hoeffding's inequality, for any $\delta\in(0,1)$, with probability at least $1-\delta$,
\begin{equation}
  \varepsilon_\theta
  \;\le\;
  \widehat{\varepsilon}_\theta
  +
  \sqrt{\frac{\log(1/\delta)}{2n}}.
  \label{eq:eps-hoeffding}
\end{equation}
To obtain a uniform bound over a finite set of checkpoints $\Theta_{\mathrm{ckpt}}$, we may set
\begin{equation}
  \varepsilon
  \;\coloneqq\;
  \sup_{\theta\in\Theta_{\mathrm{ckpt}}}\widehat{\varepsilon}_\theta
  \;+\;
  \sqrt{\frac{\log(|\Theta_{\mathrm{ckpt}}|/\delta)}{2n}},
  \label{eq:eps-choice}
\end{equation}
so that (by a union bound) $\sup_{\theta\in\Theta_{\mathrm{ckpt}}}\varepsilon_\theta\le \varepsilon$
holds with probability at least $1-\delta$.
Our theory only requires that $\varepsilon$ is finite and satisfies $\varepsilon<\tfrac12$,
i.e., drift is not the majority behavior on-policy.

\subsection{Why Actor--Editor: Actor-only vs.\ Editor-only vs.\ Actor--Editor}
\label{app:why-actor-editor}

We compare three training designs for PDDL formalization:
\emph{Actor-only}, \emph{Editor-only} (a single monolithic Editor-like policy), and \emph{Actor--Editor}.
All results are stated for deterministic planner/validator semantics; all probabilities are over
$(X,Y)$ induced by task sampling $X\sim\mathcal D$ and stochastic generation $Y\sim\pi(\cdot\mid X)$.

\paragraph{Setup.}
Let $R=r(X,Y)\in\{0,1\}$ be solver success and $C=c(X,Y)\in\{0,1\}$ be true correctness.
For any policy that produces an \emph{initial} specification $Y_0$, define
\begin{equation}
  C_{\mathrm{init}}(\pi)
  \;\coloneqq\;
  \mathbb{E}\bigl[C(X,Y_0)\bigr]\in[0,1].
  \label{eq:Cinit-A3}
\end{equation}

\begin{assumption}[Bounded solver--correctness drift rate]
\label{ass:bounded-noise-A3}
There exists $\varepsilon\in[0,1/2)$ such that for every policy $\pi$ considered in this comparison,
\begin{equation}
  \Pr_{X\sim\mathcal D,\;Y\sim\pi(\cdot\mid X)}\!\bigl(r(X,Y)\neq c(X,Y)\bigr)
  \;\le\;
  \varepsilon.
  \label{eq:drift-A3}
\end{equation}
\end{assumption}

\paragraph{A generic optimization-to-alignment lemma.}
\begin{lemma}[Generic gap bound]
\label{lem:generic-gap-A3}
Let $C(\theta)\in[0,1]$ and $J(\theta)\in\mathbb{R}$ satisfy
\begin{equation}
  \bigl|J(\theta)-(a\,C(\theta)+b)\bigr|\le \Delta
  \qquad\forall\theta,
  \label{eq:affine-A3}
\end{equation}
for some $a>0$, $b\in\mathbb{R}$, $\Delta\ge 0$.
If $\theta_J\in\arg\max_\theta J(\theta)$ and $\theta_C\in\arg\max_\theta C(\theta)$, then
\begin{equation}
  C(\theta_C)-C(\theta_J)\;\le\;\frac{2\Delta}{a}.
  \label{eq:gap-A3}
\end{equation}
\end{lemma}

\begin{proof}
From~\eqref{eq:affine-A3}, for any $\theta$,
$J(\theta)\ge aC(\theta)+b-\Delta$ and $J(\theta)\le aC(\theta)+b+\Delta$.
Using optimality of $\theta_J$ and $\theta_C$ yields
$aC(\theta_J)+b+\Delta\ge J(\theta_J)\ge J(\theta_C)\ge aC(\theta_C)+b-\Delta$,
hence $a(C(\theta_C)-C(\theta_J))\le 2\Delta$.
\end{proof}

\subsubsection*{Actor-only: solver reward aligns initial correctness on-policy}

\paragraph{Definition.}
Actor-only samples $Y_0\sim\pi_A(\cdot\mid X)$ and maximizes
\begin{equation}
  J_{\mathrm{A-only}}(\pi_A)
  \;\coloneqq\;
  \mathbb{E}\bigl[r(X,Y_0)\bigr].
  \label{eq:JAonly-A3}
\end{equation}

\begin{theorem}[Actor-only initial-alignment under drift rate]
\label{thm:actor-only-A3}
Under Assumption~\ref{ass:bounded-noise-A3}, for all $\pi_A$,
\begin{equation}
  \bigl|
    J_{\mathrm{A-only}}(\pi_A) - C_{\mathrm{init}}(\pi_A)
  \bigr|
  \;\le\;
  \varepsilon.
  \label{eq:actor-only-affine-A3}
\end{equation}
Consequently any maximizer $\pi_A^\star\in\arg\max_{\pi_A}J_{\mathrm{A-only}}(\pi_A)$ satisfies
\begin{equation}
  C_{\mathrm{init}}(\pi_A^\star)
  \;\ge\;
  \sup_{\pi_A}C_{\mathrm{init}}(\pi_A) - 2\varepsilon.
  \label{eq:actor-only-gap-A3}
\end{equation}
\end{theorem}

\begin{proof}
Since $r,c\in\{0,1\}$,
\[
  \bigl|J_{\mathrm{A-only}}(\pi_A)-C_{\mathrm{init}}(\pi_A)\bigr|
  =
  \bigl|\mathbb{E}[r(X,Y_0)-c(X,Y_0)]\bigr|
  \le
  \mathbb{E}\bigl[|r-c|\bigr]
  =
  \Pr(r\neq c)
  \le \varepsilon,
\]
proving~\eqref{eq:actor-only-affine-A3}.
Apply Lemma~\ref{lem:generic-gap-A3} with $a=1$, $b=0$, $\Delta=\varepsilon$
to obtain~\eqref{eq:actor-only-gap-A3}.
\end{proof}

\subsubsection*{Editor-only: final-outcome objective cannot constrain initial correctness}

\paragraph{Monolithic Editor-only baseline.}
A monolithic policy initializes $Y_0\sim q(\cdot\mid X)$ and iteratively repairs:
\[
  Y_{t+1}\sim\pi_E^{\mathrm{mono}}(\cdot\mid X,Y_t,F_t,t),
  \qquad t=0,\dots,T-1.
\]
The objective depends only on the final solver outcome:
\begin{equation}
  J_{\mathrm{mono}}(q,\pi_E^{\mathrm{mono}})
  \;\coloneqq\;
  \mathbb{E}\bigl[r(X,Y_T)-\beta T\bigr].
  \label{eq:Jmono-A3}
\end{equation}

\begin{assumption}[Editor-saturated repair]
\label{ass:editor-saturated-A3}
There exist sets $\mathcal{Y}_{\mathrm{good}}(x)\subseteq\mathcal Y$ and constants $T^\dagger\in\mathbb{N}$, $\beta\ge 0$ such that:
\begin{enumerate}
  \item (\emph{Non-catastrophic initializations}) $\Pr(Y_0\in\mathcal{Y}_{\mathrm{good}}(X))=1$.
  \item (\emph{Uniform repair success}) There exists a repair policy $\pi_E^\dagger$ such that for all $x$ and all $y_0\in\mathcal{Y}_{\mathrm{good}}(x)$,
  \[
    \Pr_{\pi_E^\dagger}\bigl(r(x,Y_{T^\dagger})=1\mid X=x,Y_0=y_0\bigr)=1,
    \quad\text{and}\quad T=T^\dagger\ \text{a.s.}
  \]
\end{enumerate}
\end{assumption}

\begin{assumption}[Nontrivial good-set]
\label{ass:goodset-rich-A3}
For every $x$ there exist $y^+(x),y^-(x)\in\mathcal{Y}_{\mathrm{good}}(x)$ such that
$c(x,y^+(x))=1$ and $c(x,y^-(x))=0$.
\end{assumption}

\begin{theorem}[Editor-only unidentifiability of initial correctness]
\label{thm:editor-only-A3}
Under Assumptions~\ref{ass:editor-saturated-A3} and~\ref{ass:goodset-rich-A3}, for every $\alpha\in[0,1]$ there exists a monolithic policy $(q^\alpha,\pi_E^\dagger)$ such that
\begin{equation}
  J_{\mathrm{mono}}(q^\alpha,\pi_E^\dagger)=1-\beta T^\dagger
  \quad\text{and}\quad
  C_{\mathrm{init}}(q^\alpha)=\alpha.
  \label{eq:mono-alpha-A3}
\end{equation}
In particular, $J_{\mathrm{mono}}$ admits globally optimal policies with arbitrarily low $C_{\mathrm{init}}$.
\end{theorem}

\begin{proof}
Fix $\alpha\in[0,1]$ and define $q^\alpha(\cdot\mid x)$ by
$Y_0=y^+(x)$ with probability $\alpha$ and $Y_0=y^-(x)$ with probability $1-\alpha$.
By Assumption~\ref{ass:editor-saturated-A3}, for any $y_0\in\mathcal{Y}_{\mathrm{good}}(x)$,
repair by $\pi_E^\dagger$ yields $r(x,Y_{T^\dagger})=1$ a.s.\ and $T=T^\dagger$ a.s.,
hence $\mathbb{E}[r(X,Y_T)-\beta T\mid X=x,Y_0=y_0]=1-\beta T^\dagger$ independent of $y_0$.
Taking expectation over $X$ and $Y_0\sim q^\alpha(\cdot\mid X)$ gives
$J_{\mathrm{mono}}(q^\alpha,\pi_E^\dagger)=1-\beta T^\dagger$.

Moreover, by Assumption~\ref{ass:goodset-rich-A3},
$c(X,y^+(X))=1$ and $c(X,y^-(X))=0$, so
\[
  C_{\mathrm{init}}(q^\alpha)
  =
  \mathbb{E}[c(X,Y_0)]
  =
  \mathbb{E}[\alpha\cdot 1 + (1-\alpha)\cdot 0]
  =
  \alpha.
\]
\end{proof}

\subsubsection*{Actor--Editor: restores initial alignment and improves final solvability}

\paragraph{Actor reward with a calibrated Judge.}
Actor--Editor samples $Y_0\sim\pi_A(\cdot\mid X)$ and uses an Actor reward
\begin{equation}
  R_A
  \;\coloneqq\;
  r(X,Y_0) + \lambda_J s_J(X,Y_0),
  \qquad \lambda_J\ge 0,
  \label{eq:RA-A3}
\end{equation}
with objective $J_A(\pi_A)\coloneqq\mathbb{E}[R_A]$.
We assume the Judge is calibrated to the solver label $r(X,Y)$.

\begin{assumption}[Judge calibration to the solver]
\label{ass:judge-cal-A3}
There exists $\delta_J\ge 0$ such that
\begin{equation}
  \sup_{x,y}\bigl|s_J(x,y)-r(x,y)\bigr|\le \delta_J.
  \label{eq:deltaJ-A3}
\end{equation}
\end{assumption}

\begin{lemma}[Judge is correctness-consistent on-policy]
\label{lem:judge-correctness-A3}
Under Assumptions~\ref{ass:bounded-noise-A3} and~\ref{ass:judge-cal-A3}, for any policy $\pi$ generating $(X,Y)$,
\begin{equation}
  \mathbb{E}\bigl[\,|s_J(X,Y)-c(X,Y)|\,\bigr]
  \;\le\;
  \delta_J + \varepsilon.
  \label{eq:judge-L1-A3}
\end{equation}
Moreover,
\begin{equation}
  \Pr\bigl(|s_J(X,Y)-c(X,Y)|>\delta_J\bigr)\le \varepsilon.
  \label{eq:judge-hp-A3}
\end{equation}
\end{lemma}

\begin{proof}
By triangle inequality, $|s_J-c|\le |s_J-r|+|r-c|$.
Taking expectation gives
$\mathbb{E}|s_J-c|\le \mathbb{E}|s_J-r|+\mathbb{E}|r-c|\le \delta_J+\Pr(r\neq c)\le \delta_J+\varepsilon$.
For the high-probability statement, on the event $\{r=c\}$ we have $|s_J-c|=|s_J-r|\le \delta_J$,
so $\{|s_J-c|>\delta_J\}\subseteq\{r\neq c\}$ and $\Pr(|s_J-c|>\delta_J)\le \Pr(r\neq c)\le\varepsilon$.
\end{proof}

\begin{theorem}[Actor objective aligns initial correctness in Actor--Editor]
\label{thm:actor-align-A3}
Under Assumptions~\ref{ass:bounded-noise-A3} and~\ref{ass:judge-cal-A3}, for all $\pi_A$,
\begin{equation}
  \Bigl|
    J_A(\pi_A) - (1+\lambda_J)\,C_{\mathrm{init}}(\pi_A)
  \Bigr|
  \;\le\;
  (1+\lambda_J)\varepsilon + \lambda_J\delta_J.
  \label{eq:actor-align-A3}
\end{equation}
Consequently, any maximizer $\pi_A^\star\in\arg\max_{\pi_A}J_A(\pi_A)$ satisfies
\begin{equation}
  C_{\mathrm{init}}(\pi_A^\star)
  \;\ge\;
  \sup_{\pi_A}C_{\mathrm{init}}(\pi_A)
  \;-\;
  \frac{2\bigl((1+\lambda_J)\varepsilon+\lambda_J\delta_J\bigr)}{1+\lambda_J}.
  \label{eq:actor-gap-A3}
\end{equation}
\end{theorem}

\begin{proof}
By definition,
\[
  J_A(\pi_A)
  =
  \mathbb{E}[r(X,Y_0)] + \lambda_J\,\mathbb{E}[s_J(X,Y_0)].
\]
Subtract $(1+\lambda_J)\mathbb{E}[c(X,Y_0)]$ and apply triangle inequality:
\[
  \Bigl|J_A-(1+\lambda_J)C_{\mathrm{init}}\Bigr|
  \le
  \bigl|\mathbb{E}[r-c]\bigr|
  +
  \lambda_J\bigl|\mathbb{E}[s_J-c]\bigr|
  \le
  \mathbb{E}|r-c| + \lambda_J\,\mathbb{E}|s_J-c|.
\]
Assumption~\ref{ass:bounded-noise-A3} gives $\mathbb{E}|r-c|=\Pr(r\neq c)\le\varepsilon$,
and Lemma~\ref{lem:judge-correctness-A3} gives $\mathbb{E}|s_J-c|\le \delta_J+\varepsilon$,
yielding~\eqref{eq:actor-align-A3}.
Apply Lemma~\ref{lem:generic-gap-A3} with $a=1+\lambda_J$, $b=0$, and
$\Delta=(1+\lambda_J)\varepsilon+\lambda_J\delta_J$ to obtain~\eqref{eq:actor-gap-A3}.
\end{proof}

\paragraph{Editor monotonicity for final solvability.}
Let the Editor iteratively refine $Y_t$ for $t\le T_{\max}$.
Define
\begin{equation}
  R_{\mathrm{final}}
  \;\coloneqq\;
  \mathbb{I}\{\exists t\le T_{\max}: r(X,Y_t)=1\}.
  \label{eq:Rfinal-A3}
\end{equation}

\begin{lemma}[Repair is weakly monotone in solver success]
\label{lem:repair-monotone-A3}
If the Editor action space includes a no-op (i.e., it can keep $Y_{t+1}=Y_t$),
then for any fixed Actor policy $\pi_A$ and any Editor policy $\pi_E$,
\begin{equation}
  \mathbb{E}[R_{\mathrm{final}}]
  \;\ge\;
  \mathbb{E}[r(X,Y_0)].
  \label{eq:monotone-A3}
\end{equation}
\end{lemma}

\begin{proof}
Pointwise, the event $\{r(X,Y_0)=1\}$ implies $\{\exists t\le T_{\max}: r(X,Y_t)=1\}$,
since the Editor can keep $Y_t=Y_0$.
Thus $R_{\mathrm{final}}\ge r(X,Y_0)$ almost surely, and taking expectations yields~\eqref{eq:monotone-A3}.
\end{proof}

\begin{theorem}[Why Actor--Editor]
\label{thm:why-actor-editor-A3}
Assume Assumption~\ref{ass:bounded-noise-A3}.
\begin{enumerate}
  \item (\textbf{Actor-only}) Maximizing $J_{\mathrm{A-only}}$ yields an initial-correctness guarantee within $2\varepsilon$
  of the optimum (Theorem~\ref{thm:actor-only-A3}).
  \item (\textbf{Editor-only}) In the saturated-repair regime (Assumptions~\ref{ass:editor-saturated-A3} and~\ref{ass:goodset-rich-A3}),
  the monolithic objective admits globally optimal policies with arbitrarily low $C_{\mathrm{init}}$ (Theorem~\ref{thm:editor-only-A3}).
  \item (\textbf{Actor--Editor}) With a calibrated Judge (Assumption~\ref{ass:judge-cal-A3}),
  maximizing the Actor objective $J_A$ enforces initial-correctness alignment up to
  $\mathcal{O}(\varepsilon+\delta_J)$ (Theorem~\ref{thm:actor-align-A3}),
  while the Editor weakly improves final solvability (Lemma~\ref{lem:repair-monotone-A3}).
\end{enumerate}
\end{theorem}

\subsection{Scaling of Local Reward Hacking: Separate vs.~Shared Parameters}
\label{app:reward-hacking-params}

This subsection provides a proof of the scaling claim in Theorem~3.2 with the \emph{worst-case first-order reward-hacking gain} under a local linearization.

\paragraph{Local linearization and hacking gain.}
Fix a reference point $\theta_0$ and a constant $a>0$.
Let $J_A(\theta)$ denote the Actor objective and $C(\theta)$ the (true) plan-correctness objective.
Define the \emph{reward-hacking gradient}
\begin{equation}
  g_{\mathrm{hack}}(\theta_0)
  \;\coloneqq\;
  \nabla_\theta J_A(\theta_0) - a\,\nabla_\theta C(\theta_0).
  \label{eq:ghack-def}
\end{equation}
For a perturbation $\Delta\theta$ with $\|\Delta\theta\|_2\le 1$, the first-order (local) hacking gain is
\begin{equation}
  \Delta_{\mathrm{hack}}(\theta_0;\Delta\theta)
  \;\coloneqq\;
  g_{\mathrm{hack}}(\theta_0)^\top \Delta\theta.
  \label{eq:dhack-def}
\end{equation}
Given an admissible perturbation set $\mathcal H$, the worst-case local hacking gain is
\begin{equation}
  \sup_{\Delta\theta\in\mathcal H}\Delta_{\mathrm{hack}}(\theta_0;\Delta\theta).
  \label{eq:wst-hack}
\end{equation}

\paragraph{Architectures and admissible perturbation sets.}
\begin{itemize}
\item \textbf{Separate models.}
The Actor has its own parameters $\theta_A\in\mathbb{R}^{P_A}$.
A unit-norm perturbation can move freely in the Actor space:
\begin{equation}
  \mathcal H_{\mathrm{sep}}
  \;\coloneqq\;
  \bigl\{\Delta\theta_A\in\mathbb{R}^{P_A}:\ \|\Delta\theta_A\|_2\le 1\bigr\}.
  \label{eq:Hsep}
\end{equation}

\item \textbf{Shared backbone.}
Parameters decompose as
\[
  \theta = (\theta_{\mathrm{sh}},\theta_A,\theta_J,\theta_E),
\]
where $\theta_{\mathrm{sh}}\in\mathbb{R}^{p}$ is the shared backbone, and
$\theta_A\in\mathbb{R}^{h}$ is the Actor head (with $h\ll p$).
We define admissible perturbations as those that are \emph{(Judge,Editor)-neutral on the backbone}
and do not modify the Judge/Editor heads:
\begin{equation}
  \mathcal H_{\mathrm{shr}}
  \;\coloneqq\;
  \Bigl\{
    \Delta\theta:\ \|\Delta\theta\|_2\le 1,\ 
    \Delta\theta_J=0,\ \Delta\theta_E=0,\ 
    \Delta\theta_{\mathrm{sh}}\in \mathrm{Null}(\Sigma_{JE})
  \Bigr\},
  \label{eq:Hshr}
\end{equation}
where $\Sigma_{JE}\in\mathbb{R}^{p\times p}$ is defined in Assumption~\ref{ass:rank-constraint-appendix} below.
\end{itemize}

\begin{assumption}[Judge and Editor constrain the backbone (covariance form)]
\label{ass:rank-constraint-appendix}
At $\theta_0$, let $g_J\in\mathbb{R}^{p}$ and $g_E\in\mathbb{R}^{p}$ denote the (random) backbone gradients
associated with the Judge and Editor objectives under their respective training distributions
(e.g., per-sample or per-trajectory stochastic gradients).
Define the (backbone) second-moment / Gram matrix
\begin{equation}
  \Sigma_{JE}
  \;\coloneqq\;
  \mathbb{E}\bigl[g_J g_J^\top\bigr] + \mathbb{E}\bigl[g_E g_E^\top\bigr]
  \in\mathbb{R}^{p\times p}.
  \label{eq:SigmaJE}
\end{equation}
Assume $\Sigma_{JE}$ is full rank on the backbone, equivalently
\begin{equation}
  \mathrm{Null}(\Sigma_{JE}) = \{0\}.
  \label{eq:Sigma-null}
\end{equation}
\end{assumption}

Assumption~\ref{ass:rank-constraint-appendix} is a high-rank \emph{coverage} condition:
Judge and Editor gradients collectively ``see'' all backbone directions (in second moment),
so there is no nonzero backbone direction that is simultaneously neutral to both.

\begin{assumption}[Isotropic hacking gradient]
\label{ass:isotropic-appendix}
There exists $\sigma>0$ such that the restriction of $g_{\mathrm{hack}}(\theta_0)$
to the relevant parameter subspace is isotropic Gaussian.
Concretely:
\begin{enumerate}
\item In the \emph{separate} setting, $g_{\mathrm{hack}}(\theta_0)\in\mathbb{R}^{P_A}$ satisfies
\begin{equation}
  g_{\mathrm{hack}}(\theta_0)\sim \mathcal{N}(0,\sigma^2 I_{P_A}).
  \label{eq:iso-sep}
\end{equation}
\item In the \emph{shared} setting, the projection of $g_{\mathrm{hack}}(\theta_0)$ onto the Actor-head subspace
$\mathbb{R}^{h}$ satisfies
\begin{equation}
  \mathrm{Proj}_{A}\,g_{\mathrm{hack}}(\theta_0)\sim \mathcal{N}(0,\sigma^2 I_{h}).
  \label{eq:iso-shr}
\end{equation}
\end{enumerate}
\end{assumption}

\subsubsection*{A Gaussian supremum lemma}

\begin{lemma}[Supremum of an isotropic Gaussian over a unit ball]
\label{lem:gauss-sup}
Let $g\sim\mathcal{N}(0,\sigma^2 I_d)$ in $\mathbb{R}^d$ and let
$\mathbb{B}_d\coloneqq\{u\in\mathbb{R}^d:\|u\|_2\le 1\}$.
Then
\begin{equation}
  \mathbb{E}\Big[\sup_{u\in\mathbb{B}_d} g^\top u\Big]
  =
  \mathbb{E}\big[\|g\|_2\big]
  =
  \sigma\,\mathbb{E}\big[\|Z\|_2\big],
  \qquad Z\sim\mathcal{N}(0,I_d).
  \label{eq:gauss-sup-eq}
\end{equation}
Moreover, for all $d\ge 1$,
\begin{equation}
  \sigma\sqrt{d-1}
  \;\le\;
  \mathbb{E}\big[\|g\|_2\big]
  \;\le\;
  \sigma\sqrt{d},
  \label{eq:gauss-norm-bounds}
\end{equation}
and in particular $\mathbb{E}[\|g\|_2]=\Theta(\sigma\sqrt d)$.
\end{lemma}

\begin{proof}
For any fixed $g$, the Cauchy--Schwarz inequality gives
\(
  \sup_{\|u\|\le 1} g^\top u = \|g\|_2
\),
achieved by $u=g/\|g\|_2$ if $g\neq 0$.
Taking expectation yields the first equality in~\eqref{eq:gauss-sup-eq}.

The upper bound in~\eqref{eq:gauss-norm-bounds} follows from Jensen:
\(
  \mathbb{E}\|g\|_2 \le \sqrt{\mathbb{E}\|g\|_2^2} = \sqrt{\mathbb{E}[g^\top g]}=\sigma\sqrt d.
\)
The lower bound $\mathbb{E}\|Z\|_2\ge \sqrt{d-1}$ for $Z\sim\mathcal{N}(0,I_d)$ is classical
(equivalently for a $\chi_d$ random variable), and follows for example from standard gamma-function
bounds for $\mathbb{E}\chi_d=\sqrt{2}\,\Gamma\bigl(\frac{d+1}{2}\bigr)/\Gamma\bigl(\frac{d}{2}\bigr)$.
Multiplying by $\sigma$ yields~\eqref{eq:gauss-norm-bounds}.
\end{proof}

\subsubsection*{Main result}

\begin{theorem}[Scaling of local reward-hacking gain: separate vs.~shared]
\label{thm:hacking-scaling}
Under Assumptions~\ref{ass:rank-constraint-appendix} and~\ref{ass:isotropic-appendix}:

\begin{enumerate}
\item \textbf{Three separate models.}
With $\mathcal H_{\mathrm{sep}}$ as in~\eqref{eq:Hsep},
\begin{equation}
  \mathbb{E}\Bigl[
    \sup_{\Delta\theta\in\mathcal H_{\mathrm{sep}}}
    \Delta_{\mathrm{hack}}(\theta_0;\Delta\theta)
  \Bigr]
  \;=\;
  \Theta\!\bigl(\sigma\sqrt{P_A}\bigr).
  \label{eq:sep-scaling}
\end{equation}

\item \textbf{One shared backbone.}
With $\mathcal H_{\mathrm{shr}}$ as in~\eqref{eq:Hshr},
\begin{equation}
  \mathbb{E}\Bigl[
    \sup_{\Delta\theta\in\mathcal H_{\mathrm{shr}}}
    \Delta_{\mathrm{hack}}(\theta_0;\Delta\theta)
  \Bigr]
  \;\le\;
  \sigma\sqrt{h},
  \label{eq:shr-scaling}
\end{equation}
hence the expected local reward-hacking gain is $\mathcal{O}(\sigma\sqrt h)$ and does not scale with
the backbone dimension $p$.
\end{enumerate}
\end{theorem}

\begin{proof}
We prove the two items.

\paragraph{(1) Separate models.}
In the separate setting, $\Delta\theta$ ranges over the Actor parameter space $\mathbb{R}^{P_A}$,
and by~\eqref{eq:dhack-def}--\eqref{eq:Hsep},
\[
  \sup_{\Delta\theta\in\mathcal H_{\mathrm{sep}}}\Delta_{\mathrm{hack}}(\theta_0;\Delta\theta)
  =
  \sup_{\|\Delta\theta\|_2\le 1} g_{\mathrm{hack}}(\theta_0)^\top \Delta\theta
  =
  \|g_{\mathrm{hack}}(\theta_0)\|_2.
\]
By Assumption~\ref{ass:isotropic-appendix} (separate case) and Lemma~\ref{lem:gauss-sup} with $d=P_A$,
\[
  \mathbb{E}\Bigl[
    \sup_{\Delta\theta\in\mathcal H_{\mathrm{sep}}}\Delta_{\mathrm{hack}}(\theta_0;\Delta\theta)
  \Bigr]
  =
  \mathbb{E}\|g_{\mathrm{hack}}(\theta_0)\|_2
  =
  \Theta(\sigma\sqrt{P_A}),
\]
proving~\eqref{eq:sep-scaling}.

\paragraph{(2) Shared backbone.}
By Assumption~\ref{ass:rank-constraint-appendix}, $\mathrm{Null}(\Sigma_{JE})=\{0\}$, hence
$\Delta\theta_{\mathrm{sh}}=0$ for all $\Delta\theta\in\mathcal H_{\mathrm{shr}}$ by~\eqref{eq:Hshr}.
Moreover $\Delta\theta_J=\Delta\theta_E=0$ by definition of $\mathcal H_{\mathrm{shr}}$.
Therefore any $\Delta\theta\in\mathcal H_{\mathrm{shr}}$ lies entirely in the Actor-head subspace,
so we can write $\Delta\theta=(0,\Delta\theta_A,0,0)$ with $\|\Delta\theta_A\|_2\le 1$.
Thus,
\[
  \sup_{\Delta\theta\in\mathcal H_{\mathrm{shr}}}\Delta_{\mathrm{hack}}(\theta_0;\Delta\theta)
  =
  \sup_{\|\Delta\theta_A\|_2\le 1}
  \bigl(\mathrm{Proj}_{A}g_{\mathrm{hack}}(\theta_0)\bigr)^\top \Delta\theta_A
  =
  \bigl\|\mathrm{Proj}_{A}g_{\mathrm{hack}}(\theta_0)\bigr\|_2.
\]
By Assumption~\ref{ass:isotropic-appendix} (shared case),
$\mathrm{Proj}_{A}g_{\mathrm{hack}}(\theta_0)\sim\mathcal{N}(0,\sigma^2 I_h)$.
Applying Lemma~\ref{lem:gauss-sup} with $d=h$ yields
\[
  \mathbb{E}\Bigl[
    \sup_{\Delta\theta\in\mathcal H_{\mathrm{shr}}}\Delta_{\mathrm{hack}}(\theta_0;\Delta\theta)
  \Bigr]
  =
  \mathbb{E}\bigl\|\mathrm{Proj}_{A}g_{\mathrm{hack}}(\theta_0)\bigr\|_2
  \le
  \sigma\sqrt{h},
\]
which is~\eqref{eq:shr-scaling}. The bound depends on $h$ but not on $p$.
\end{proof}

\paragraph{Remark (effective-rank variant).}
If $\mathrm{Null}(\Sigma_{JE})$ has dimension $k>0$ (i.e., Judge+Editor do not constrain $k$ backbone directions),
then the same proof yields a shared-architecture bound of order $\sigma\sqrt{h+k}$, where $k$ is the number of
unconstrained backbone directions. This makes explicit that the scaling depends on the \emph{dimension of the
(Judge,Editor)-neutral subspace}, not directly on the backbone size $p$.

\section{Detailed Experimental Setup}
\label{app:experimental-details}

\subsection{Benchmarks}

\paragraph{PlanBench.}
PlanBench~\citep{valmeekam2024planbench} evaluates whether language models can solve classical planning problems described in natural language while preserving the underlying symbolic transition structure. Each instance is derived from a PDDL planning problem and converted into a natural-language prompt. We report results on four domains: BlocksWorld, Mystery BlocksWorld, Logistics, and Gripper. BlocksWorld tests basic object manipulation and state-transition reasoning. Mystery BlocksWorld preserves the same transition dynamics but systematically renames objects and predicates, thereby removing lexical cues and testing whether the model recovers the underlying symbolic structure. Logistics and Gripper introduce more brittle object typing, movement constraints, and action-schema dependencies.

\paragraph{Zero-shot transfer benchmarks.}
To evaluate whether solver-grounded training improves planning behavior beyond the PlanBench domains, we additionally test zero-shot transfer on ProntoQA~\citep{saparov2023languagemodelsgreedyreasoners} and NATURAL PLAN~\citep{zheng2024naturalplan}. ProntoQA requires multi-hop logical inference from explicitly stated facts and rules. NATURAL PLAN evaluates realistic natural-language planning under constraints; we use the Trip Planning and Calendar Scheduling subsets. In NATURAL PLAN, all external information needed for planning, such as flight connectivity or calendar availability, is provided in-context, so performance reflects planning and constraint satisfaction rather than tool invocation.

\subsection{Evaluation Protocol}

\paragraph{PDDL-based evaluation.}
For PlanBench, each model output is parsed as a PDDL domain--problem pair. We use Fast Downward with a 60-second timeout as the external planner. An instance is counted as successful only when both generated files are syntactically valid, the planner returns a plan within the timeout, and the plan satisfies the stated goal conditions. Outputs that fail parsing, grounding, type checking, or goal achievement are counted as failures.

\paragraph{Non-PDDL evaluation.}
For ProntoQA and NATURAL PLAN, we follow the original benchmark protocols and report exact-match success against the gold label or gold plan. We apply the same answer-normalization rules across all methods. No benchmark-specific supervised examples are used for training our method.

\subsection{Model and Training Setting}

\paragraph{Backbone.}
Our framework uses Qwen2.5-7B~\citep{qwen2025qwen25technicalreport} initialized from pretrained weights. The model is conditioned into three roles: Actor, Judge, and Editor. The Actor generates the initial PDDL specification, the Judge predicts a solver-calibrated quality score, and the Editor repairs failed specifications using solver diagnostics.

\paragraph{Annotation-free training.}
No human-written PDDL demonstrations, gold domain files, gold problem files, or supervised planning traces are used to train our method. The only external feedback used during training is produced by the symbolic planner and its diagnostics. This ensures that the reported performance reflects solver-grounded learning rather than supervised imitation of annotated PDDL.

\paragraph{Inference.}
At inference time, the Actor first generates an initial specification. If the specification fails solver verification, the Editor performs bounded diagnostic-conditioned repair. Unless otherwise specified, we use the same maximum repair horizon as in training. The system returns the first solver-executable specification found within the repair budget; if no repair succeeds, the final edited specification is submitted for evaluation.

\subsection{Baselines}

\paragraph{Chain-of-Thought prompting.}
The CoT baseline uses chain-of-thought prompting~\citep{wei2022chain} to elicit intermediate reasoning before producing the final plan or formal specification. We use an out-of-domain prompt and do not provide task-specific PDDL annotations.

\paragraph{Tree-of-Thought search.}
The ToT baseline follows the search-based reasoning framework of~\citet{yao2023tree}. The language model expands and evaluates intermediate reasoning states using breadth-first search. The search process does not access symbolic execution feedback during thought expansion.

\paragraph{LLM+P.}
LLM+P~\citep{liu2023llmp} translates natural-language planning descriptions into PDDL, invokes a classical planner, and then maps the resulting plan back to natural language when needed. In our setup, the baseline uses fixed domain assumptions and in-context PDDL examples, but no additional human-annotated training data.

\paragraph{Backbone choice.}
The main comparison uses GPT-4o for CoT, ToT, and LLM+P, providing strong off-the-shelf zero-annotation baselines. To separate algorithmic gains from backbone effects and solver-call budgets, we additionally report matched-backbone and matched-budget comparisons in Appendix~\ref{sec:fair-baselines}, where all controlled baselines use Qwen2.5-7B and the same maximum number of solver calls as our method.

\section{Additional Experimental Details and Discussion}
\label{app:additional-experimental-details}
\subsection{Fair Baseline Comparison under Matched Backbone and Solver Budget}
\label{sec:fair-baselines}

The main comparison in Table~\ref{tab:main-and-transfer} follows prior planning baselines with strong prompting and neuro-symbolic pipelines. 
However, to rule out confounding factors from different backbone models and different verifier-access budgets, we conduct an additional controlled comparison. 
All methods in this subsection use the same Qwen2.5-7B backbone as our method and are allowed the same maximum number of solver calls at inference time. 
Since our default inference procedure consists of one initial Actor generation followed by at most five Editor repairs, the maximum solver-call budget is
\[
  K = H_{\max}+1 = 6 .
\]
For prompting baselines that do not use solver diagnostics, we allow up to \(K\) independently decoded candidates and select the first solver-successful output if one exists. 
For the self-refine baseline, the model receives the raw solver diagnostic after each failed attempt and revises its previous PDDL specification for at most five rounds. 
This gives all baselines comparable access to solver verification while isolating the effect of trained role specialization.

\paragraph{Controlled baselines.}
We evaluate four same-backbone baselines:
\begin{itemize}[leftmargin=*]
    \item \textbf{Qwen-CoT} uses chain-of-thought prompting with up to six sampled candidates. 
    The solver is used only for candidate selection, not for textual feedback.
    \item \textbf{Qwen-ToT} uses tree-of-thought search with Qwen2.5-7B as both generator and evaluator. 
    We cap the final solver-checked candidates at six.
    \item \textbf{Qwen-LLM+P} prompts Qwen2.5-7B to translate the natural-language task into PDDL and invokes the planner on each candidate. 
    We allow up to six independently sampled formalizations.
    \item \textbf{Qwen-Self-Refine+Solver} uses the same backbone and the same solver-call budget as our method, but has no trained Actor, Judge, or Editor. 
    It revises its previous PDDL output using the solver diagnostic through prompting alone.
\end{itemize}
None of these baselines uses supervised PDDL annotations or task-specific training.

\begin{table}[!htbp]
\centering
\small
\caption{
Planning success rate (\%) under matched backbone and matched solver-call budget. 
All Qwen baselines use Qwen2.5-7B and at most \(K=6\) solver calls per test instance. 
Avg. denotes the mean across the four PlanBench domains.
}
\label{tab:fair-baselines}
\resizebox{\linewidth}{!}{
\begin{tabular}{@{}lccccccc@{}}
\toprule
\textbf{Method} 
& \textbf{Backbone}
& \textbf{Max Calls}
& \textbf{BW}
& \textbf{MBW}
& \textbf{Logistics}
& \textbf{Gripper}
& \textbf{Avg.} \\
\midrule
Qwen-CoT 
& Qwen2.5-7B 
& 6 
& 31 
& 1 
& 5 
& 25 
& 15.5 \\
Qwen-ToT 
& Qwen2.5-7B 
& 6 
& 16 
& 5 
& 9 
& 18 
& 12.0 \\
Qwen-LLM+P 
& Qwen2.5-7B 
& 6 
& 80 
& 32 
& 20 
& 7 
& 34.8 \\
Qwen-Self-Refine+Solver 
& Qwen2.5-7B 
& 6 
& 85 
& 43 
& 36 
& 35 
& 49.8 \\
\midrule
GPT-4o-Self-Refine+Solver 
& GPT-4o 
& 6 
& 91 
& 47 
& 42 
& 37 
& 54.3 \\
\midrule
\textbf{Ours} 
& \textbf{Qwen2.5-7B} 
& \textbf{6} 
& \textbf{98} 
& \textbf{71} 
& \textbf{58} 
& \textbf{56} 
& \textbf{70.8} \\
\bottomrule
\end{tabular}
}
\end{table}

Table~\ref{tab:fair-baselines} shows that the advantage of our method is not caused by using a stronger backbone or by having more access to the symbolic solver. 
When all methods use Qwen2.5-7B and the same maximum solver-call budget, CoT and ToT remain weak, especially on Mystery BlocksWorld and Logistics. 
This indicates that sampling more reasoning traces or searching over textual thoughts does not reliably recover valid symbolic transition models. 
Qwen-LLM+P performs well on BlocksWorld but degrades sharply on Mystery BlocksWorld, Logistics, and Gripper, suggesting that in-context PDDL formalization remains sensitive to domain familiarity and brittle action-schema construction.

The strongest same-backbone baseline is Qwen-Self-Refine+Solver, which uses the same solver diagnostics available to our Editor at inference time. 
It improves over Qwen-LLM+P by revising syntactic and type-level mistakes, reaching 49.8\% average success. 
However, it still remains 21.0 percentage points behind our method. 
This gap suggests that solver diagnostics alone are useful but insufficient: without a trained Judge and diagnostic-conditioned Editor, prompting-based repair often fixes local errors while leaving global predicate structure, goal preservation, or action semantics unstable. 
The GPT-4o self-refine baseline provides an even stronger untrained reference, but it still underperforms our Qwen2.5-7B method. 
Therefore, the improvement is better explained by the learned multi-role training procedure rather than by backbone strength or additional solver calls.

\begin{table}[!htbp]
\centering
\small
\caption{
Budget efficiency and semantic faithfulness under the matched \(K=6\) solver-call setting. 
Avg. Calls is the average number of solver invocations actually used before success or budget exhaustion. 
Faithful denotes the fraction of all tasks that are both solver-successful and semantically faithful according to the reference checker. 
Drift is the conditional fraction of solver-successful outputs that fail semantic checking.
}
\label{tab:budget-faithfulness}
\begin{tabular}{@{}lcccc@{}}
\toprule
\textbf{Method} 
& \textbf{Avg. Calls} 
& \textbf{Solv. (\%)} 
& \textbf{Faithful (\%)} 
& \textbf{Drift (\%)} $\downarrow$ \\
\midrule
Qwen-CoT 
& \(5.8 \pm 0.4\) 
& \(15.5 \pm 2.5\) 
& \(11.6 \pm 2.0\) 
& \(25.2 \pm 3.0\) \\
Qwen-ToT 
& \(5.9 \pm 0.4\) 
& \(12.0 \pm 2.0\) 
& \(9.1 \pm 1.5\) 
& \(24.4 \pm 3.0\) \\
Qwen-LLM+P 
& \(5.2 \pm 0.3\) 
& \(34.8 \pm 3.5\) 
& \(27.1 \pm 3.0\) 
& \(22.1 \pm 2.5\) \\
Qwen-Self-Refine+Solver 
& \(4.7 \pm 0.3\) 
& \(49.8 \pm 4.0\) 
& \(37.5 \pm 3.5\) 
& \(24.7 \pm 3.0\) \\
GPT-4o-Self-Refine+Solver 
& \(4.5 \pm 0.3\) 
& \(54.3 \pm 4.0\) 
& \(43.0 \pm 3.5\) 
& \(20.8 \pm 2.5\) \\
\midrule
\textbf{Ours} 
& \(\mathbf{3.2 \pm 0.2}\) 
& \(\mathbf{70.8 \pm 4.5}\) 
& \(\mathbf{66.3 \pm 4.0}\) 
& \(\mathbf{6.4 \pm 1.0}\) \\
\bottomrule
\end{tabular}
\end{table}

Table~\ref{tab:budget-faithfulness} further shows that our method is not simply spending the solver budget more aggressively. 
Although all methods are allowed up to six solver calls, our method uses only 3.2 calls on average because many instances are solved by the initial Actor output or by early Editor repairs. 
In contrast, CoT, ToT, and LLM+P frequently exhaust the budget without producing executable specifications. 
Self-refine baselines use diagnostics more effectively, but their semantic drift remains high: many repaired specifications become solver-executable without fully preserving the intended task semantics. 
Our method has the smallest gap between solvability and faithful success, and its drift rate is substantially lower than all matched-budget baselines. 
This supports the default design choice: the learned Judge and Editor do not merely increase the chance of passing the solver, but also stabilize the symbolic specification so that solver success remains aligned with task-level faithfulness.

These controlled comparisons address two possible alternative explanations. 
First, the gain is not due to backbone choice, since all same-backbone Qwen2.5-7B baselines remain below our method. 
Second, the gain is not due to a larger verifier budget, since every method is capped at the same \(K=6\) solver calls, and our method uses fewer calls on average. 
The remaining performance gap is therefore attributable to the trained multi-role decomposition: the Actor learns global PDDL construction, the Judge suppresses solver-facing shortcuts, and the Editor learns targeted diagnostic-conditioned repair.

\subsection{Discussion of Main Planning and Transfer Results}
\label{app:main-result-discussion}

Table~\ref{tab:main-and-transfer} shows that LLM+P performs strongly on BlocksWorld but drops sharply on Logistics and Gripper. 
This pattern is consistent with the fact that BlocksWorld is a canonical planning domain frequently appearing in PDDL tutorials and planning examples, whereas Logistics and Gripper require more careful object typing, action preconditions, and movement constraints. 
Mystery BlocksWorld is especially diagnostic because it preserves the same transition dynamics as BlocksWorld while removing lexical cues through systematic renaming. 
The large gain on Mystery BlocksWorld therefore suggests that our method is not simply relying on familiar predicate names, but is learning a solver-grounded mapping from task descriptions to symbolic structure.

For zero-shot transfer, ProntoQA evaluates multi-hop logical inference from explicit facts and rules, while NATURAL PLAN evaluates realistic planning tasks expressed in natural language. 
The Trip Planning and Calendar Scheduling subsets require satisfying multiple constraints using information provided in-context. 
The improvement on these benchmarks indicates that the learned roles transfer beyond PDDL syntax: the Actor learns to propose structured solutions, the Judge learns to score consistency, and the Editor learns to refine outputs under constraint feedback.

\subsection{Controlled Baseline Protocol}
\label{app:controlled-protocol}

The controlled comparison in Table~\ref{tab:controlled-comparison-main} is designed to isolate the effect of trained role specialization. 
All methods are allowed at most
\[
K = H_{\max}+1 = 6
\]
solver calls per instance, matching our default inference procedure: one Actor generation followed by at most five Editor repairs.

For CoT, ToT, and LLM+P variants without diagnostic repair, we allow up to \(K\) independently decoded candidate outputs and select the first solver-successful specification. 
For self-refine baselines, the model receives the raw solver diagnostic after each failed attempt and revises the previous PDDL specification for up to five rounds. 
This gives prompting-based repair access to the same type of solver feedback available to our Editor at inference time, but without any role-specific training.

Average solver calls are computed as the number of solver invocations used before either the first successful specification is found or the budget is exhausted. 
The faithful and drift columns are computed using the same post-hoc reference checker described in Appendix~\ref{app:faithfulness-protocol}. 
The result that our method uses \(3.2\) calls on average while achieving the highest faithful success indicates that the trained Actor and Editor reduce both search cost and semantic drift.

\subsection{Semantic Faithfulness Protocol}
\label{app:faithfulness-protocol}

The standard PlanBench metric evaluates whether the generated PDDL specification is solver-executable. 
However, solver success alone does not guarantee that the generated specification preserves the original natural-language task. 
A model may weaken goals, omit constraints, alter object types, or distort action schemas while still producing a solvable PDDL instance. 
We therefore evaluate semantic faithfulness with a post-hoc reference checker that is used only for evaluation and never for training.

For each task \(x_i\), let
\[
y_i=(y_i^{\mathrm{dom}},y_i^{\mathrm{prob}})
\]
be the generated PDDL specification and let \(\hat{\pi}_i\) be the plan returned by Fast Downward when \(y_i\) is solvable. 
The solver-success indicator is
\begin{equation}
  S_i =
  \mathbf{1}\{\mathcal{E}(y_i)=1\}.
\end{equation}
The reference checker \(\mathcal{V}_{\mathrm{ref}}\) verifies three conditions:
\begin{enumerate}[leftmargin=*]
    \item \textbf{Goal preservation:} generated goal conditions are semantically equivalent to the reference goal after canonicalizing object and predicate names;
    \item \textbf{Object and type preservation:} generated objects and type assignments preserve the entity structure of the original task;
    \item \textbf{Action-schema consistency:} generated action schemas preserve the reference transition semantics, measured by canonical precondition/effect matching and plan replay under the reference transition model.
\end{enumerate}
The faithful-success indicator is then
\begin{equation}
  C_i =
  \mathbf{1}\{S_i=1 \wedge \mathcal{V}_{\mathrm{ref}}(x_i,y_i,\hat{\pi}_i)=1\}.
\end{equation}
We report
\begin{equation}
  \mathrm{Solvability}
  =
  \frac{1}{N}\sum_{i=1}^{N} S_i,
  \qquad
  \mathrm{FaithfulSuccess}
  =
  \frac{1}{N}\sum_{i=1}^{N} C_i,
\end{equation}
and the conditional drift rate
\begin{equation}
  \mathrm{Drift}
  =
  \frac{\sum_{i=1}^{N} S_i(1-C_i)}
       {\sum_{i=1}^{N} S_i}.
  \label{eq:semantic-drift-app}
\end{equation}
A lower drift rate means that fewer solver-successful outputs are achieved through semantic shortcuts.

\subsection{Discussion of Faithfulness and Drift Results}
\label{app:faithfulness-discussion}

The faithfulness results in Table~\ref{tab:faithfulness-main} show that solver-only RL increases raw solvability but also creates substantial specification drift. 
This confirms the central concern that raw solver success is an exploitable proxy when the model controls the generated formal specification. 
In contrast, the full model improves solvability and faithful success simultaneously, while maintaining a much lower drift rate.

The per-domain faithful-success results show that the advantage is not restricted to easier domains. 
On Mystery BlocksWorld, where lexical cues are removed, our method retains \(67\%\) faithful success. 
On Logistics and Gripper, faithful success remains lower because small mistakes in typing, movement constraints, or precondition/effect structure can distort the intended transition system. 
Nevertheless, our method remains substantially better than LLM+P and solver-only variants, suggesting that the multi-role design improves specification quality rather than merely increasing solver executability.

\subsection{Ablation and Drift-Type Discussion}
\label{app:ablation-drift-discussion}

The ablation results in Table~\ref{tab:ablation-drift-main}(a) isolate the contribution of each role on Mystery BlocksWorld. 
The Actor-only setting collapses because the policy must search a large symbolic specification space using sparse solver feedback. 
Adding the Judge without the Editor improves success but remains limited because the system lacks a targeted repair mechanism. 
Adding the Editor without the Judge also fails, suggesting that diagnostic repair alone can still drift toward solver-facing shortcuts. 
The full model succeeds because the Actor handles global formalization, the Judge supplies a calibrated quality signal, and the Editor performs bounded local correction.

The separate-model baseline with \(3\times\)7B parameters also collapses. 
This result is consistent with the analysis in Section~\ref{sec:sharing-theory}: fully separate role models expose more role-private directions through which the Actor can improve solver-facing reward without being constrained by Judge calibration or Editor repair behavior. 
The shared-backbone design reduces these degrees of freedom and stabilizes cross-role credit assignment.

Table~\ref{tab:ablation-drift-main}(b) categorizes the remaining unfaithful but solver-successful outputs. 
LLM+P and solver-only RL contain many direct shortcut failures, especially goal weakening and object/type drift. 
Our method reduces these shortcut-like errors substantially. 
Among the few remaining drift cases, action-schema drift becomes the largest category, indicating that the residual failures are mostly subtle transition-model mistakes rather than systematic goal removal or object manipulation.

\subsection{Repair and Diagnostic Protocol}
\label{app:repair-protocol}

For each test instance, the Actor first generates an initial PDDL specification \(y_0\), which is verified by the solver. 
If verification fails, the Editor performs up to \(H_{\max}=5\) sequential repairs, each conditioned on the task, the current specification, and the latest solver diagnostic. 
Let \(S_i^{(h)}\in\{0,1\}\) denote whether instance \(i\) has been solved by step \(h\), where \(h=0\) corresponds to the initial Actor output and \(h\in\{1,\dots,5\}\) denotes the number of Editor repairs. 
We report the cumulative success rate
\begin{equation}
\mathrm{CSR}(h)
=
\frac{1}{N}\sum_{i=1}^{N} S_i^{(h)},
\qquad h=0,1,\dots,5,
\end{equation}
and the marginal repair gain
\begin{equation}
\Delta(h)
=
\mathrm{CSR}(h)-\mathrm{CSR}(h-1),
\qquad h\ge 1.
\end{equation}

Diagnostics are grouped into six categories:
\textit{syntax/parse error},
\textit{object/type mismatch},
\textit{predicate/schema mismatch},
\textit{precondition/effect mismatch},
\textit{goal/constraint drift},
and \textit{unreachable planning search}. 
The diagnostic heatmap in Figure~\ref{fig:repair-diagnostic-main} reports the dominant diagnostic type among remaining unsolved cases after each repair step.

Averaged across domains, the initial Actor solves \(46.2\%\) of instances. 
Cumulative success rises to \(56.5\%\), \(63.5\%\), \(67.5\%\), \(69.3\%\), and \(70.8\%\) after one to five repairs. 
Thus, the Editor contributes \(24.6\) absolute points beyond the initial Actor output. 
The first two repairs account for \(17.3\) points, or about \(70\%\) of the total repair gain, which supports using a bounded repair horizon rather than an unbounded search loop.

\subsection{Solvability--Faithfulness Gap Discussion}
\label{app:gap-discussion}

The solvability--faithfulness gap is defined as
\begin{equation}
  \mathrm{Gap}
  =
  \mathrm{Solvability}
  -
  \mathrm{FaithfulSuccess}.
  \label{eq:solvability-faithfulness-gap-app}
\end{equation}
This metric measures how tightly solver success aligns with task-level semantic correctness. 
A large gap means that many solver-successful outputs are semantically unfaithful.

Figure~\ref{fig:solvability-faithfulness-gap-main} shows that our method has the smallest gap. 
The Qwen-Self-Refine+Solver baseline improves raw solvability by using solver diagnostics, but its faithful success remains much lower, indicating that prompting-based repair often fixes local syntax or type errors without preserving the deeper transition semantics. 
The GPT-4o self-refine baseline has a similar issue: stronger language modeling improves executable output generation, but does not fully suppress specification drift. 
The smaller gap of our method suggests that the Actor, Judge, and Editor jointly maintain a closer connection between executable PDDL and the original natural-language task.

\subsection{Judge Candidate Construction and Interpretation}
\label{app:judge-protocol}

To evaluate whether the Judge distinguishes faithful outputs from solver-successful but unfaithful outputs, we collect a candidate pool from model-generated specifications during inference. 
The pool includes initial Actor outputs, intermediate Editor repairs, and final returned specifications. 
Each candidate \(y\) is assigned two post-hoc labels:
\[
  S(y)\in\{0,1\},
  \qquad
  C(y)\in\{0,1\},
\]
where \(S(y)=1\) means the specification is solver-executable, and \(C(y)=1\) means the specification is both solver-executable and semantically faithful under the reference checker.

We partition candidates into three groups:
\begin{equation}
\begin{aligned}
  \mathcal{U} &= \{y:S(y)=0\}, \\
  \mathcal{D} &= \{y:S(y)=1,\ C(y)=0\}, \\
  \mathcal{F} &= \{y:S(y)=1,\ C(y)=1\},
\end{aligned}
\end{equation}
corresponding to unsolved, solved-but-unfaithful, and solved-and-faithful outputs. 
For each candidate, we record the Judge score \(s_J(x,y)\in[0,1]\).

If the Judge only predicted raw solver acceptance, solved-but-unfaithful and faithful outputs would receive similar scores because both pass the solver. 
Instead, Figure~\ref{fig:judge-drift-violin-main} and Table~\ref{tab:judge-drift-main} show a clear ordering:
\[
\mathcal{U} < \mathcal{D} < \mathcal{F}.
\]
This indicates that the Judge captures structural quality signals correlated with semantic faithfulness, even though it is not trained with human semantic labels. 
The separation is strongest on BlocksWorld and weaker on Logistics and Gripper, where subtle action-schema errors can still pass solver checks under distorted specifications.


\end{document}